\documentclass[sigconf,nonacm]{aamas} 

\PassOptionsToPackage{numbers, compress}{natbib}
\usepackage[disable=True]{todonotes}
\newcommand{\piotrm}[1]{\textcolor{blue}{}}
\newcommand{\michalz}[1]{\textcolor{green!50!black}{}}
\newcommand{\blazej}[1]{\textcolor{magenta}{}}
\newcommand{\reviewers}[1]{\textcolor{violet}{}}

\usepackage{amsmath}
\DeclareMathOperator{\E}{\mathbb{E}}

\usepackage{graphicx}
\usepackage{appendix}
\usepackage{numprint}
\usepackage[final]{pdfpages}
\usepackage{xcolor}
\usepackage{comment}

\newcommand{\mt}{MA-Trace}

\newcommand{\vt}{V-Trace}
\newcommand{\imatrace}{DecMA-Trace}

\newtheorem{theorem}{Theorem}

\newtheorem{corollary}[theorem]{Corollary}
\newtheorem{remark}[theorem]{Remark}

\usepackage{balance}        
\usepackage[utf8]{inputenc} 
\usepackage[T1]{fontenc}    
\usepackage{hyperref}       
\usepackage{url}            
\usepackage{booktabs}       
\usepackage{amsfonts}       
\usepackage{nicefrac}       
\usepackage{microtype}      
\usepackage{xcolor}         
\usepackage{csvsimple}      
\usepackage{subcaption}     

\usepackage{algorithmicx}
\usepackage{algorithm}
\usepackage[noend]{algpseudocode}

\usepackage{wrapfig}
\usepackage{graphicx}

\usepackage{amsmath,amsfonts,amsthm}

\setcopyright{ifaamas}
\acmConference[AAMAS '22]{Proc.\@ of the 21st International Conference
on Autonomous Agents and Multiagent Systems (AAMAS 2022)}{May 9--13, 2022}
{Auckland, New Zealand}{P.~Faliszewski, V.~Mascardi, C.~Pelachaud,
M.E.~Taylor (eds.)}
\copyrightyear{2022}
\acmYear{2022}
\acmDOI{}
\acmPrice{}
\acmISBN{}

\acmSubmissionID{636}

\title{Off-Policy Correction For Multi-Agent Reinforcement Learning}

\author{Michał Zawalski}
\affiliation{
  \institution{University of Warsaw}
  \city{Warsaw}
  \country{Poland}}
\email{m.zawalski@uw.edu.pl}

\author{Błażej Osiński}
\affiliation{
  \institution{University of Warsaw}
  \city{Warsaw}
  \country{Poland}}
\email{b.osinski@mimuw.edu.pl}

\author{Henryk Michalewski}
\affiliation{
  \institution{Google Research}
  \country{}}
\email{henrykm@google.com}

\author{Piotr Miłoś}
\affiliation{
  \institution{Polish Academy of Sciences}
  \city{Warsaw}
  \country{Poland}}
\email{pmilos@impan.pl}

\begin{abstract}
Multi-agent reinforcement learning (MARL) provides a framework for problems involving multiple interacting agents.
Despite apparent similarity to the single-agent case, multi-agent problems are often harder to train and analyze theoretically.
In this work, we propose \mt{}, a new on-policy actor-critic algorithm, which extends \vt{} to the MARL setting.
The key advantage of our algorithm is its high scalability in a multi-worker setting.
To this end, \mt{} utilizes importance sampling as an off-policy correction method, which allows distributing the computations with no impact on the quality of training.
Furthermore, our algorithm is theoretically grounded -- we prove a fixed-point theorem that guarantees convergence.
We evaluate the algorithm extensively on the StarCraft Multi-Agent Challenge, a standard benchmark for multi-agent algorithms.
\mt{} achieves high performance on all its tasks and exceeds state-of-the-art results on some of them.
\end{abstract}

\keywords{Reinforcement Learning, V-Trace, Importance Sampling, Scalability}

\newcommand{\BibTeX}{\rm B\kern-.05em{\sc i\kern-.025em b}\kern-.08em\TeX}

\begin{document}


\pagestyle{fancy}
\fancyhead{}


\maketitle

\section{Introduction}
Reinforcement learning has witnessed impressive development in recent years. Famously, superhuman performance has been achieved in games Go \cite{alphago}, StarCraft II \cite{alphastar}, Dota 2 \cite{openaifive} and other applications. These successes are the result of rapid algorithmic development. Research in directions like trust-region optimization \cite{DBLP:journals/corr/SchulmanWDRK17}, principle of maximum entropy \cite{DBLP:conf/icml/HaarnojaZAL18}, importance sampling \cite{DBLP:conf/icml/EspeholtSMSMWDF18}, distributional RL \cite{DBLP:conf/icml/BellemareDM17} or bridging the sim-to-real gap \cite{DBLP:journals/ijrr/OpenAI20} are among these which brought significant progress. Multi-agent reinforcement learning (MARL), a framework for problems involving multiple interacting agents, is similar to the standard, single-agent setting. However, it is inherently harder. The challenges are both theoretical (e.g., partial observability and lack of the Markov property) and practical (MARL algorithms often suffer from inferior stability and scalability).

In this work, we take a step towards amending this situation. We propose \mt{}, a new on-policy actor-critic algorithm, which adheres to the centralized training and decentralized execution paradigm \cite{DBLP:conf/nips/LoweWTHAM17, 6_DBLP:conf/aaai/FoersterFANW18, 4_DBLP:conf/icml/RashidSWFFW18}. The key component of \mt{} is the usage of importance sampling. This mechanism, based on \vt{} \cite{DBLP:conf/icml/EspeholtSMSMWDF18}, provides off-policy correction for training data. As we demonstrate empirically, it allows distributing the computations efficiently in a multi-worker setup. Another advantage of \mt{} is the fact that it is theoretically grounded. We provide a fixed-point theorem that guarantees convergence.

The on-policy algorithms directly optimize the objective; thus, they tend to be more stable and robust to hyperparameter choices than off-policy methods \cite{spinningup, DBLP:journals/tac/TsitsiklisR97, sutton2018reinforcement}. However, it is often impractical to train an on-policy algorithm in the distributed setting.
When data collection is performed using many workers, the communication latency, asynchronicity, and other factors make the behavioral policies lag behind the target one. This results in a shift of the collected data towards off-policy distribution, which hurts the training quality. \vt{} reduces this shift by utilizing importance weights, thus permitting highly scalable training. Following that scheme, \mt{} can be distributed to many workers to vastly reduce the wall-time of training with no negative impact on the results.

We evaluate \mt{} on StarCraft Multi-Agent Challenge \cite{DBLP:conf/atal/SamvelyanRWFNRH19} -- a standard benchmark for multi-agent algorithms. Our approach achieves competitive performance on all tasks and exceeds state-of-the-art results on some of them. Additionally, we provide a comprehensive set of ablations to quantify the influence of each component on the final results. We confirm that importance sampling is a key factor for \mt's performance
and show that our algorithm scales favorably with the number of actor workers. Additionally, we provide a few quite surprising findings, e.g. that an observation-based critic network performs better than a state-based.

For the description of key ideas and videos, visit our webpage: \url{https://sites.google.com/view/ma-trace/main-page}. The code used for our experiments is available at \url{https://github.com/awarelab/seed_rl}. 

Our main contributions are the following:
\begin{enumerate}
    \item We introduce \mt{} -- a simple, scalable and effective multi-agent reinforcement learning algorithm with theoretical guarantees.
    \item We confirm that the training of \mt{} can be easily distributed on multiple workers with nearly perfect speed-up and no negative impact on the quality.
    \item We provide extensive experimental validation of the \mt{} algorithm in StarCraft Multi-Agent Challenge, including ablations with regard to importance sampling, centralization of learning, scaling and sharing of parameters.
\end{enumerate}

\section{Related work}

For a general overview of multi-agent reinforcement learning (MARL) we refer to \cite{busoniu2008comprehensive, 1_DBLP:conf/atal/Hernandez-LealK20}. Unsurprisingly, the development of MARL methods is closely coupled with the algorithmic progress in RL. A simple approach to multi-agent learning was proposed by \citet{DBLP:conf/icml/Tan93}: the IQL algorithm uses independent $Q$ learners for each agent, with improvements proposed in \cite{DBLP:conf/icml/FoersterNFATKW17, DBLP:conf/icml/LauerR00, omidshafiei2017deep}.

\mt{} adheres to the centralized training and decentralized execution (CTDE) paradigm. CTDE \cite{kraemer2016multi, DBLP:series/sbis/OliehoekA16, DBLP:conf/nips/FoersterAFW16} is based on using the centralized information during training. During execution, the agents act using only their respective observations. Following this scheme, \cite{DBLP:conf/nips/FoersterAFW16} introduces the RIAL and DIAL algorithms in the context of $Q$-learning. CTDE is particularly easy to implement with actor-critic algorithms; the centralized information is imputed only to the critic network (which is not used during the execution). COMA \cite{6_DBLP:conf/aaai/FoersterFANW18} is an example of such an algorithm; additionally, it uses a counterfactual baseline to deal with multi-agent credit assignment explicitly.

Another approach to take advantage of the multi-agent structure is the \textit{value decomposition} method. VDN \cite{sunehag2017valuedecomposition} propose a linear decomposition of the collective $Q$ function into agent-local $Q$ functions.
Following this idea, \cite{4_DBLP:conf/icml/RashidSWFFW18} introduced QMIX, which learns a complex state-dependent decomposition by using monotonic mixing hypernetworks. Extensions of QMIX include MAVEN \cite{DBLP:conf/nips/MahajanRSW19}, COMIX \cite{de2020deep}, SMIX($\lambda$) \cite{DBLP:conf/aaai/WenYWT20}, and QTRAN \cite{DBLP:conf/icml/SonKKHY19} that can represent even general non-monotonic factorizations.

\mt{} is based on \vt{} \cite{DBLP:conf/icml/EspeholtSMSMWDF18}, a distributed single-agent algorithm. The idea of extending RL algorithms to the multi-agent setting has been successfully executed multiple times. \citet{DBLP:conf/nips/LoweWTHAM17} propose a multi-agent actor-critic algorithm MADDPG, which is based on the DDPG algorithm \cite{lillicrap2019continuous}. \citet{yu2021benchmarking} introduce also MASAC, extending SAC \cite{haarnoja2018soft}, and MATD3 building on top of TD3 \cite{DBLP:conf/icml/FujimotoHM18}. Recently \cite{DBLP:journals/corr/abs-2103-01955} showed that MAPPO, a multi-agent version of PPO \cite{DBLP:journals/corr/SchulmanWDRK17}, achieves surprisingly strong results in the most popular benchmarks, comparable with off-policy methods. 

\citet{DBLP:conf/icml/EspeholtSMSMWDF18} propose the \vt{} algorithm to address the problem that in distributed (e.g. multi-node) training the policy used to generate experience is likely to lag behind the policy used for learning. \citet{munos2016safe} considered earlier a similar off-policy corrections for the target of the $Q$-function. These corrections are intended to focus on samples generated by behavioral policies close to the target one. Leaky \vt{}, a more general version of the \vt{} correction, was considered by \citet{DBLP:conf/nips/ZahavyXVHOHSS20}. \citet{vinyals2019grandmaster} adapt \vt{} importance corrections to large action space to train grandmaster level StarCraft II agents. These corrections are refinements of the concept of importance sampling; see \cite[Sections 5.5, 12.9]{sutton2018reinforcement} for a broader discussion.

Blending all these concepts, DOP \cite{DBLP:journals/corr/abs-2007-12322} utilizes value decomposition and importance sampling to successfully train decentralized agents with policy gradients on off-policy samples. This is substantially different from our work since in \mt{} we use importance weights to enable efficient multi-node training. DOP does not consider distributing the computations, the objective optimized by that algorithm requires providing on-policy samples, which is impossible to satisfy in a highly distributed setting.

\section{Background}

Multi-agent reinforcement learning task is formalized by \textit{decentralized partially observable Markov decision processes} (Dec-POMDP)~\cite{DBLP:series/sbis/OliehoekA16}. A Dec-POMDP is defined as a tuple $(\mathcal N, \mathcal S, \mathcal A, P, r, \mathcal Z, O, \gamma, \rho_0)$. 
$\mathcal N$ is the set of agents $\{1,\ldots,n\}$, $\mathcal S$ is the state space, $\mathcal{A}$ is the set of actions available to agents, $P$ is the transition kernel, $r$ is the reward function, $\mathcal Z$ is the space of collective observations, $O$ is the set of observation functions $\{O_1,\ldots,O_n\}$, $\gamma$ is the discount factor and $\rho_0$ is the initial state distribution. At state $s\in \mathcal{S}$, the agents select actions $a_i \sim \pi_i(\cdot|O_i(s))$, where $\pi_i$ are their respective polices. Fix $a:= (a_1, \ldots, a_n)$. The agents receive rewards according to the reward functions $r_i=r_i(s,a)$ and the system evolves to the next step generated by $P(s,a)$ (might be stochastic). In the so-called fully cooperative setting, assumed in this work, the rewards are equal, i.e. $r_1=\ldots=r_n$.

The agents learn a joint policy
\begin{equation} \label{eq:factorised_policy}
\pi(a|o) = \prod_{i=1}^n \pi_i(a_k|o_k)
\end{equation}
with the aim to maximize the expected discounted return
$$J(\pi)=\E_{\pi}\left[\sum_{t=0}^\infty \gamma^tr_t\right].$$
The expected discounted return obtained by policy $\pi$ starting from state $s\in\mathcal S$ is called the value function
\begin{equation} \label{eq:value_function}
V^{\pi}(s) = \E_{\pi}\left[\sum_{t=0}^{\infty} \gamma^t r_t |s_0 = s\right],
\end{equation}

\section{\mt{} algorithm}\label{section:matrace}

\subsection{Overview of the algorithm}
In this work, we introduce a multi-agent actor-critic algorithm based on V-trace: \textbf{\mt{}}, see Algorithm \ref{algo:mt}. It follows the paradigm of centralized training, decentralized execution. Each agent compute its action taking its local observation as input.
On the other hand, the critic network operates only during training, so it does not need to obey decentralization requirements. Furthermore, it can utilize any kind of additional information.
We study two versions of \mt{}: with the critic $V:\mathcal{S}\to \mathbb{R}$ taking as inputs full states, and with the critic $V:\mathcal{Z} \to \mathbb{R}$ taking as input the joint observation of all agents, denoted respectively as \mt{} (full) and \mt{} (obs).
\mt{} (full) requires collecting states $s_t$ in line 6 of Algorithm \ref{algo:mt} and using them as input to $V_\phi$ in lines 10 and 14. 

The value function $V^\pi$ corresponding to policy $\pi$ can be, for example, obtained by repeated application of the Bellman operator. This requires on-policy data.
The central innovation of \vt{} in the single-player setting and \mt{} in the multi-player setting is to allow for slightly off-policy data by utilizing importance sampling. To this end, we use the \vt{}-inspired policy evaluation operator $\mathcal{R}$, defined as
\begin{equation}\label{eq:vtrace_v1}
\begin{split}
	\mathcal{R} V(s) &:= V(s) + \\
	 &\E_{\mu}{\left[\sum_{t=0}^{+\infty}{\gamma^t(c_0 \cdots c_{t-1})\rho_t(r_t + \gamma V(s_{t+1}) - V(s_t))}| s_0 = s\right]},
\end{split}
\end{equation}
where $c_t= c(s_t, a_t), \rho_t:=\rho(s_t, a_t)$ are importance sampling corrections and 
$c:\mathcal{S}\times \mathcal{A}\to \mathbb{R}_+ , \rho: \mathcal{S}\times \mathcal{A}\to \mathbb{R}_+$ are measurable functions. In our algorithms we specialize to
\begin{equation}\label{eq:corrections}
	c_t := \min \left(\overline{c}, \frac{\pi(a_t|s_t)}{\mu(a_t|s_t)}\right), \quad \rho_t := \min \left(\overline{\rho}, \frac{\pi(a_t|s_t)}{\mu(a_t|s_t)}\right),
\end{equation}
where $\mu$ is a policy that collected the data and $\overline{c}, \overline{\rho}$ are hyperparameters (usually set to $1.0$). Intuitively speaking, $c_t$ controls the speed of training and $\rho_t$ balances the learned value function between $V^\pi$ and $V^\mu$. These parameters are further discussed in Corollary \ref{cor:fixed_point_corollary}. The operator $\mathcal{R}$ leads to $n$-step Monte-Carlo target $v_t$ given the state $s_t$
\begin{equation}\label{eq:mtrace_target_1}
    v_t := V(s_t) + \sum_{u=t}^{t+n+1}\gamma^{u-t}\left(\prod_{i=t}^{u-1}c_i\right)\rho_u\big(r_u + \gamma V(s_{u+1}) - V(s_u)\big).
\end{equation}
It is a random variable; the clipping with the $\min$ function in \eqref{eq:corrections} is instrumental to reducing its variance and thus making it applicable in learning. A key advantage of \mt{} is a significant reduction in the wall-time due to distributed data collection (see line 6 in Algorithm \ref{algo:mt}). From the algorithmic standpoint, the major problem to address is that the policy used for collection $\pi_{\theta'}$ might be outdated due to communication overheads. This is successfully achieved with the importance correction mechanisms described above. 

We use the communication model proposed in \cite[Figure 1, Figure 3]{espeholt2019seed}. It consists of a single learner and actor workers. The actors are simple loops around the environment, generating observations (and rewards) transmitted to the learner. The learner makes inferences (and sends back actions); moreover, it handles trajectory accumulation and training.

\begin{algorithm}[h]
  \begin{tabular}{ l c l }
    \textbf{Require: }
    & $d$ & density of training \\
    & $\alpha$ & learning rate \\

\end{tabular}
\begin{algorithmic}[1]
\For{$k$ \textbf{in} $0,\ldots,n-1$}
\State $\omega_k \gets$ random actor parameters
\EndFor
\State $\phi \gets$ random critic parameters
\For{$epoch$ \textbf{in} $0,1,2,\ldots$}
\State $\mathcal{D} \gets \emptyset$
\State add trajectories $\{\tau_i\}$ sampled with $\pi_{\theta'}$ to $\mathcal{D}$ 
\State  \Comment{collected by multiple workers possibly remote. }
\For{$i$ \textbf{in} $0, \ldots, d-1$}
\State sample $s_t \sim \mathcal{D}$
\State $\phi \gets \phi - \alpha \nabla_\phi \left[ \|v_t - V_\phi(s_t) \|^2_2 \right]$ \Comment{$v_t$ is calculated according to \eqref{eq:mtrace_target_1}}
\EndFor
\For{$i$ \textbf{in} $0, \ldots, d-1$} \label{ln:actor-begin}
\State sample $s_t \sim \mathcal{D}$
\For{$k$ \textbf{in} $1,\ldots, n$}
\State $g_k \gets {\rho_t} \nabla_{\omega} \log(\pi_{\omega}(a_{t, k}|s_{t, k}))(r_{t,k } + \gamma V_\phi^k(s_{t + 1}) - V^k_\phi(s_t))$ 
\State \Comment{$\rho_t$ is calculated according to \eqref{eq:corrections}}
\State $\omega_k \gets \theta_k + \alpha g_k$ 
\EndFor
\EndFor
\EndFor
\end{algorithmic}
\caption{\mt{}}
\label{algo:mt}
\end{algorithm}

\subsection{Theoretical analysis of \mt{}}\label{section:theoretical_analysis}
The operator, $\mathcal{R}$ enjoys the fixed point property. We present a proof of the following Theorem in 
Appendix \ref{sec:proof}. 

\begin{theorem}\label{thm:fixed_point_theorem_abstract}
	Let $c, \rho$ be such that for any $s \in \mathcal{S}, a\in \mathcal{A}$
\begin{equation} \label{eq:alphadefintion_1}
  \rho(s,a)-c(s,a)\E_{a'\sim\mu(\cdot|s')}\left[\rho(s',a')\right]\geq0,
\end{equation}
where $s'$ is the state obtained from $s$ after issuing action $a$.  Assume also that $\E_{\mu}\rho_{0}\geq\beta\in(0,1]$. Then the operator $\mathcal{R}$ is a $\mathcal{C}_{\infty}$ contraction with a unique fixed point $V^{\tilde{\pi}}$ which is a value function of a policy $\tilde{\pi}$ given by
\begin{equation}\label{eq:policy_after_corrections}
  \tilde{\pi}(a|s):=\frac{\rho(s,a)\mu(a|s)}{\sum_{b\in\mathcal{A}}\rho(s,b)\mu(b|s)}. 
\end{equation}
The contraction constant is smaller than $1-(1-\gamma)\beta<1$.
\end{theorem}

\begin{remark}
  The theorem is an extended version of \cite[Theorem 1]{DBLP:conf/icml/EspeholtSMSMWDF18}. First, we assume the vectorized statement, which is natural for the multi-agent setting. Second, the condition \eqref{eq:alphadefintion_1} admits more general importance sampling weights. We also fix a mathematical inaccuracy present in the original proof of \cite[Theorem 1]{DBLP:conf/icml/EspeholtSMSMWDF18}, see Remark \ref{eq:original_proof_glitch}. 
\end{remark}

Now we can easily show that the result follows for the importance weights used in our work.

\begin{corollary}\label{cor:fixed_point_corollary}
	Let $c_t, \rho_t$ be importance sampling weights \eqref{eq:corrections} and $0\leq \overline{c} \leq \overline{\rho}$. Assume also that $\E_\mu \rho_0 \geq \beta \in (0,1]$. Then the operator $\mathcal{R}$ is a $\mathcal{C}_\infty$ contraction with a unique fixed point $V^{\tilde{\pi}}$ which is a value function of a policy $\tilde{\pi}$ given by
	\[
		\tilde{\pi}(a|s) := \frac{\min\left(\overline{\rho}\mu(a|s), \pi(a|s)\right)}{\sum_{b\in \mathcal{A}}\min\left(\overline{\rho}\mu(b|s), \pi(b|s)\right)}.
	\]
	The contraction constant is smaller than $1-(1-\gamma)\beta<1$.
\end{corollary}

\begin{proof}
   It is easy to check that $c(s,a)=\min\left(\overline{c},\frac{\pi(a|s)}{\mu(a|s)}\right)$,
   $\rho(s,a)=\min\left(\overline{\rho},\frac{\pi(a|s)}{\mu(a|s)}\right)$ yield the importance sampling weights \eqref{eq:corrections}. Moreover, for any $s'\in \mathcal{S}$
   \begin{align*}
    \E_{a'\sim\mu(\cdot|s')}\left[\rho(s',a')\right]&=\E_{a'\sim\mu(\cdot|s')}\min\left(\overline{\rho},\frac{\pi(a'|s')}{\mu(a'|s')}\right)\\
    &\leq\E_{a\sim\mu(\cdot|s')}\left(\frac{\pi(a'|s')}{\mu(a'|s')}\right)\\
    &=1
    \end{align*}
    and therefore \eqref{eq:alphadefintion_1} holds whenever $0\le\overline{c}\leq\overline{\rho}$. Now the result follows by Theorem \ref{thm:fixed_point_theorem_abstract}. 
\end{proof} 

Observe that when $\overline{\rho}$ is infinite, the fixed point $V^{\tilde{\pi}}$ corresponds to the target policy $\pi$. On the other hand, when $\overline{\rho}$ tends to 0, $V^{\tilde{\pi}}$ gets close to the value of the behavioral policy $\mu$. In the general case, when $\overline{\rho}$ is finite but positive, the fixed point is the value function of a policy located somewhere between $\pi$ and $\mu$. However, $\tilde\pi$ does not depend on $c_t$ -- these weights affect the speed of convergence only \cite{DBLP:conf/icml/EspeholtSMSMWDF18}.

\begin{remark} \label{rem:theoretical_issues} There is a theoretical difference between \mt{} (full) and \mt{} (obs), which perhaps is subtle in some cases. The Markov property is a key element required in the proof of Corollary \ref{cor:fixed_point_corollary}. While it is by definition true for \mt{} (state) it might fail for \mt{} (obs) - if the concatenated observations do not provide a sufficient statistic of $s_t$. 
\end{remark}

For the actor network we use policy gradient updates. Here we also need importance sampling to correct for using the off-policy behavioral policy $\mu$. We recall the factorization \eqref{eq:factorised_policy}; analogously we denote joint decentralized parameterized policy $\pi_{\omega}(a_1, \ldots, a_K|s) = \prod_{k=1}^K \pi_{\omega}(a_i|s_i)$. The policy gradient theorem suggests the ascent in the direction:
\[
	g:=\E_{{a}_t \sim \pi_{\omega}} \left[\nabla_\omega \log \pi_{\omega}({a}_t|s_t) A^{\pi_{\omega}}(s_t, {a}_t))\right],
\]
where ${a}_t = (a_{t, 1}, \ldots, a_{t, K})$ and $A^{\pi_{\omega}}$ is some advantage estimator. For off-policy data collected with $\mu$ we have 
\[
	g \approx \E_{{a}_t \sim \mu} \left[\rho_t\nabla_\omega \log \pi_{\omega}({a}_t|s_t) A^{\pi_{\omega}}(s_t, {a}_t))\right],
\]
where the equality holds if $\overline{c} = +\infty$ in \eqref{eq:corrections}. This formula leads to the practical Monte-Carlo estimator used in line 14 of Algorithm \ref{algo:mt}:
\begin{equation}\label{eq:policy_gradient}
{\rho_t} (\nabla_{\omega_i} \log(\pi_{\omega_i}(a_{t, i}|s_{t, i})))(r_{t} + \gamma V(s_{t + 1}) - V(s_t)).	
\end{equation}

\section{Experiments}
\subsection{Environment}
We evaluate \mt{} on StarCraft Multi-Agent Challenge (SMAC) \cite{samvelyan19smac} (version 4.10), which is a standard benchmark for multi-agent algorithms, based on a popular real-time strategy game StarCraft~II.
It provides 14 micromanagement tasks of varying difficulty and structure.  

The aim is to win a battle against a built-in AI by using your team of agents. In easier tasks, often rudimentary coordination is enough. However, harder tasks involve engaging a stronger enemy (e.g., having more units), which requires inventing smart techniques and tricks. Each unit has a limited line of sight, which makes the environment partially observable. We provide more details in Appendix \ref{sec:smac_appendix}.

\subsection{Main result}\label{section:results}

In Table \ref{tab:results} and Figure \ref{fig:bar_results}, we present the results of the main version of our algorithm -- \mt{} (obs), in which the critic uses stacked observations of agents as described in Appendix \ref{sec:smac_appendix}. \mt{} (obs) reaches competitive results and in some cases exceeds the current state-of-the-art. We compare with a selection of the state-of-the-art algorithms on SMAC following \cite{DBLP:journals/corr/abs-2103-01955} and \cite{DBLP:journals/corr/abs-2003-08839}. We also demonstarte scalability, which lets \mt{} can reach good performance even using short training (wall time).

\begin{table}[h!] 
    \footnotesize
    \centering
    \begin{filecontents*}{eval_results.csv}
Task,MA-Trace (our),MAPPO,IPPO,QMIX,COMA,IQL
2s3z,99 \tiny{[99.5; 99.7]},100,100,95,43,75
3s5z,97 \tiny{[96.6; 98.6]},100,100,88,1,10
1c3s5z,100 \tiny{[99.8; 100]},100,100,96,31,21
5m_vs_6m,78 \tiny{[74.3; 78.4]},89,87,75,1,49
10m_vs_11m,96 \tiny{[86.2; 97.7]},97,93,95,7,34
27m_vs_30m,99 \tiny{[99; 100]},94,69,39,0,0
3s5z_vs_3s6z,87 \tiny{[81.6; 92.1]},84,83,83,0,0
MMM2,98 \tiny{[98; 98.8]},90,87,87,0,0
2s_vs_1sc,99 \tiny{[99; 99.6]},100,100,97,98,100
3s_vs_5z,0 \tiny{[0; 0]},100,100,98,0,45
6h_vs_8z,85 \tiny{[71; 88.8]},88,84,9,0,0
bane_vs_bane,100 \tiny{[100; 100]},100,100,100,64,99
2c_vs_64zg,98 \tiny{[98; 98.5]},100,98,92,0,7
corridor,91 \tiny{[88.6; 96.1]},100,98,84,0,0
\end{filecontents*}
\csvautobooktabular[respect underscore]{eval_results.csv}
    \bigskip
    \caption{\small Median win rate of \mt{} (obs) compared with other algorithms. In \textit{3s\_vs\_5z}, our agent discovers that keeping the opponents alive leads to higher rewards than killing them. This strategy, however, yields a low win rate. See Appendix~\ref{section:3svs5z} for a detailed study.}
    \label{tab:results}
\end{table}

\begin{figure}[h!]
    \centering
    \includegraphics[width=\linewidth]{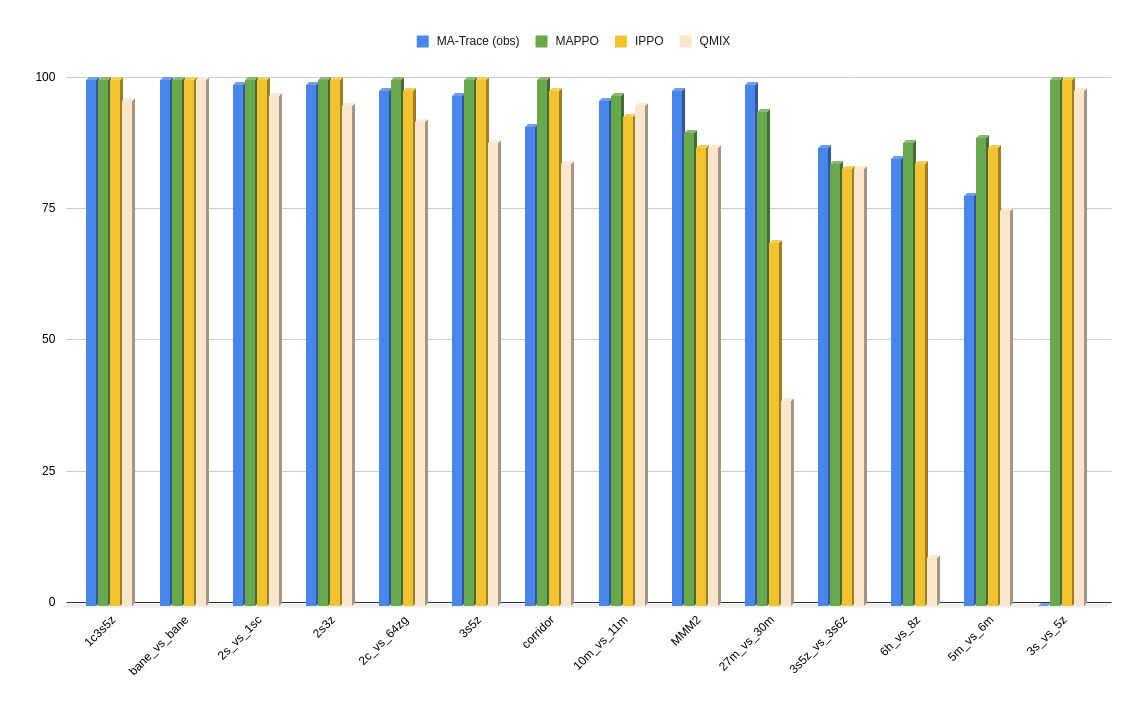}
    \caption{\small \mt{} compared with state-of-the-arts algorithms on SMAC.}
    \label{fig:bar_results}
\end{figure}

\subsection{Training details}\label{section:training_details}
We use standard feed-forward networks for the actor and critic networks with two hidden layers of $64$ neurons and ReLU activations. The critic network of \mt{} (obs) takes stacked observations of agents as input, while \mt{} (full) utilizes the full state provided by SMAC. \imatrace{} have a critic using single-agent observations. See details in Appendix \ref{sec:architectures_appendix}.

For each reported version of \mt{}, we have searched for the best hyperparameters to ensure a fair comparison.
The values of all hyperparameters are listed in Appendix \ref{sec:hiperparemters}.

To estimate the performance of \mt{} we run training for $3$ days or until convergence. We report the median win rate of $10$ runs (with different random seeds) along with the interquartile range. Training curves for all the tasks can be found in Appendix~\ref{section:all_training_data}.

\subsection{Ablations}
Below we present a comprehensive list of ablations to evaluate the design choices of our algorithm. In each case, we present training curves for tasks, which best illustrate our claims. For the complete training results and more details, we refer to Appendix~\ref{sec:ablations_appendix}.

\paragraph{Advantage of using importance sampling.}
Using the importance weights is the key algorithmic innovation of \mt{} (and \vt{}), responsible for the strong performance we report. Indeed, already for $30$ actor workers, using the weights is essential. Otherwise, the algorithm is unstable and suffers from poor asymptotic performance. See Figure \ref{fig:IS_ablation} and Appendix \ref{sec:is_ablation_appendix}.

\begin{figure}[h!]
    \centering
    \includegraphics[width=\linewidth]{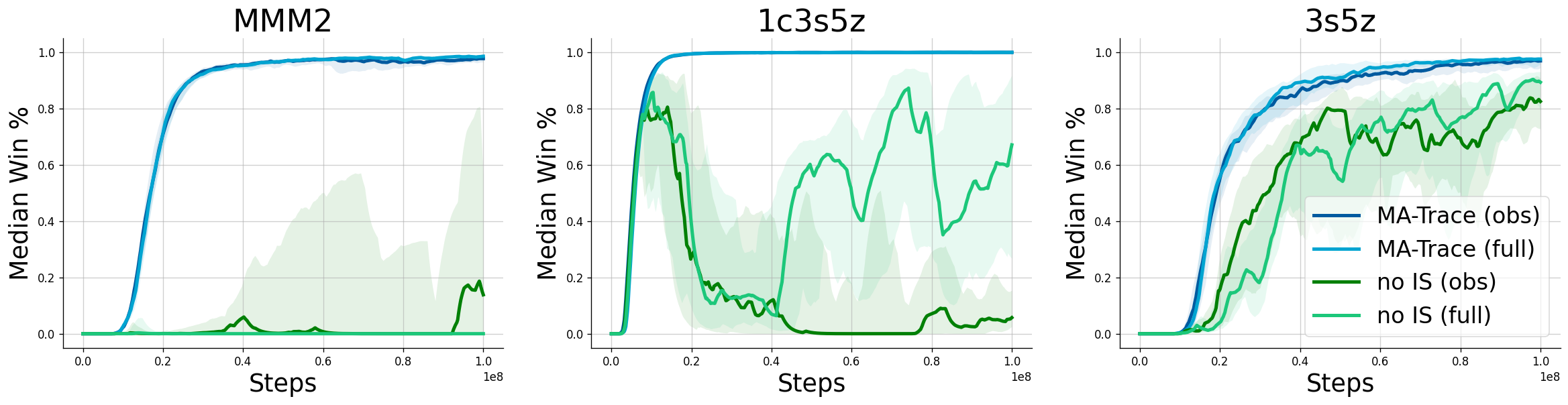}
    \caption{\small \mt{} using 30 distributed workers with and without importance sampling (no IS). }
    \label{fig:IS_ablation}
\end{figure}

\paragraph{Training scaling.}
The importance sampling enables \vt{} to be truly scalable in multi-node setups. \mt{} enjoys the same property. Importantly, we do not observe any degradation in the training performance when trained in the multi-node setup. See Figure \ref{fig:scaling} and Appendix \ref{sec:scaling_experiments}.

\begin{figure}[h!]
    \centering
    \includegraphics[width=0.6\linewidth]{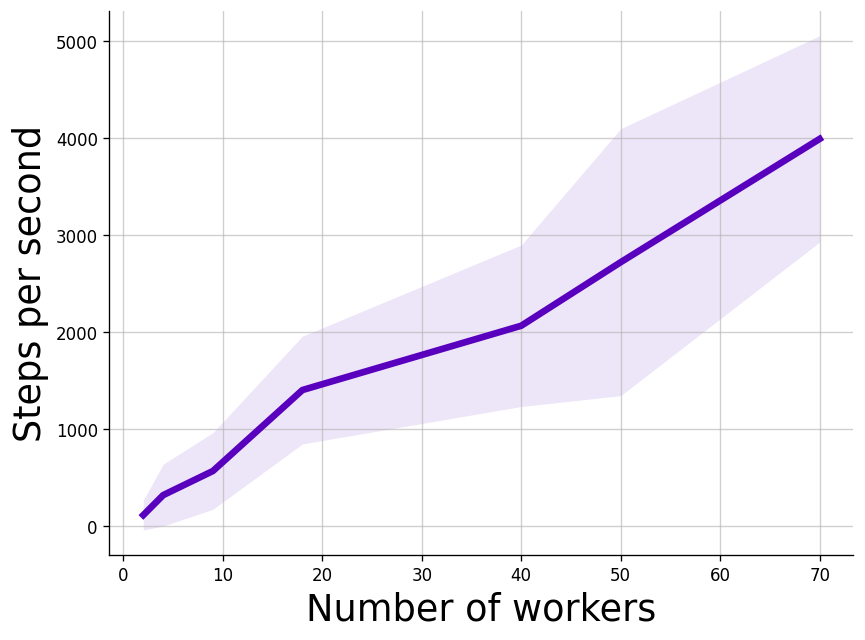}
    \caption{\small Speed of \mt{} training with respect to the number of distributed workers, with standard deviation shaded. The speed is measured as the average number of steps processed per second.}
    \label{fig:scaling}
\end{figure}

\paragraph{Input for the critic network.} We found that \mt{} (full) performs slightly worse than \mt{} (obs). Usually the differences are small. However, in two harder tasks, \textit{corridor} and \textit{6h\_vs\_8z}, \mt{} (full) learns much slower and often fails. This is perhaps surprising, as the full state contains additional information (e.g., about invisible opponents). To deepen the analysis, we ran \mt{} (obs+full), which uses both the observations and full state as the critic input. This improves the results, though they are still slightly inferior to \mt{} (obs); see Figure \ref{fig:critic_ablation} and Appendix \ref{section:all_training_data}. A more detailed discussion of this topic can be found in Appendix~\ref{sec:critic_comparison_appendix}. 

\begin{figure}[h!]
    \centering
    \includegraphics[width=\linewidth]{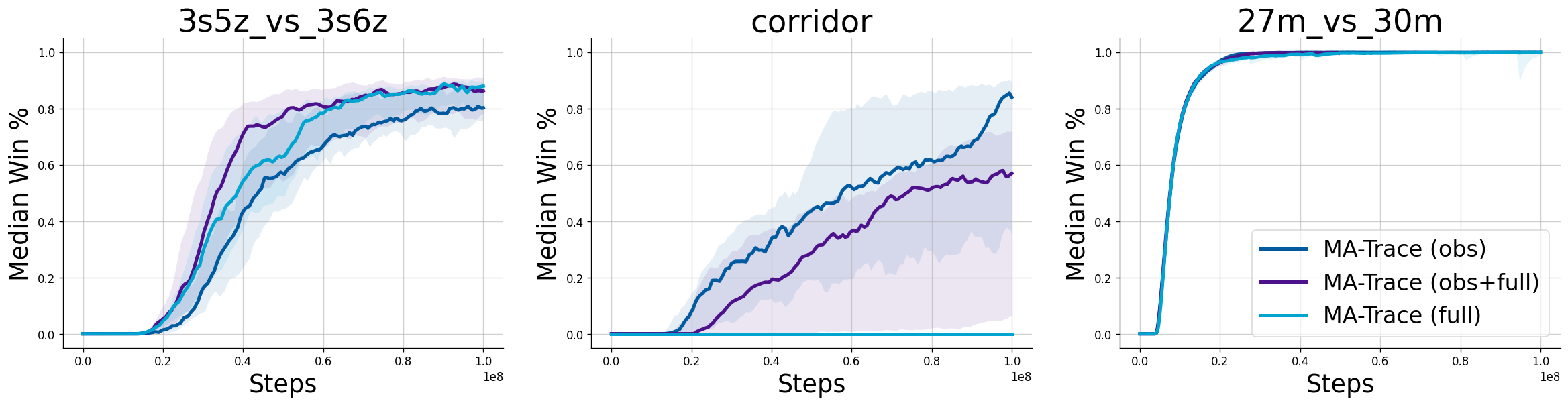}
    \caption{\small Comparison of using the full state \mt{} (full) and aggregated agents' observations \mt{} (obs) and both.}
    \label{fig:critic_ablation}
\end{figure}

\paragraph{Centralized vs decentralized.}
As noted by \cite{lyu2021contrasting}, centralized training in some cases may suffer from higher variance. Therefore we compared \mt{} with its decentralized version (i.e. having indpendent critics for each agent). The latter typically obtains weaker results and is less stable. See Figure \ref{fig:decentralized_ablation} and details in Appendix \ref{sec:decentralized_appendix}.

\begin{figure}[h!]
    \centering
    \includegraphics[width=\linewidth]{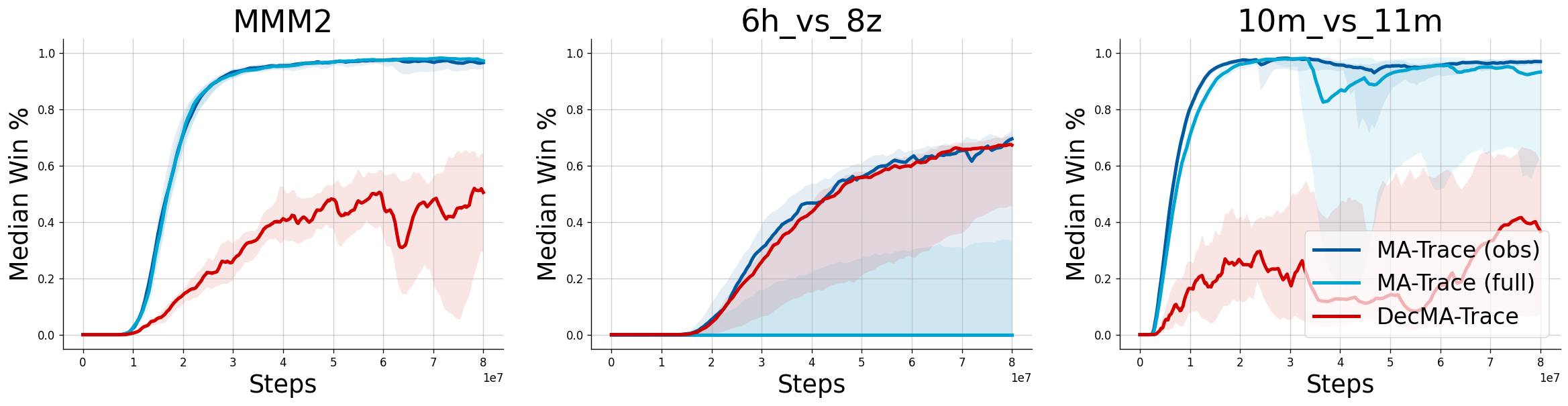}
    \caption{\small Performance of \mt{} during centralized and decentralized training.} 
    \label{fig:decentralized_ablation}
\end{figure}

\paragraph{Sharing actors networks.}
We follow a common approach of sharing the policy network between agents. In some works, e.g., \cite{DBLP:journals/corr/abs-2003-08839}, to preserve individuality, the observations are enriched with agent ID. This might be beneficial if agents should be assigned different roles within the team. However, we find these benefits rather minor and opt for input provided by the environment (i.e., without ID). See Figure \ref{fig:ID_ablation} and details in Appendix \ref{sec:id_experiments_appendix}.

One can also use separate networks for each agent. We check that \mt{} works considerably worse in such a case. In rare cases, using separate networks is advantageous, but only in the easiest tasks, e.g., \textit{3s5z}. See Figure \ref{fig:separate_ablation} and details in Appendix \ref{sec:separate_networks_exps}.

\begin{figure}[h!]
    \centering
    \includegraphics[width=\linewidth]{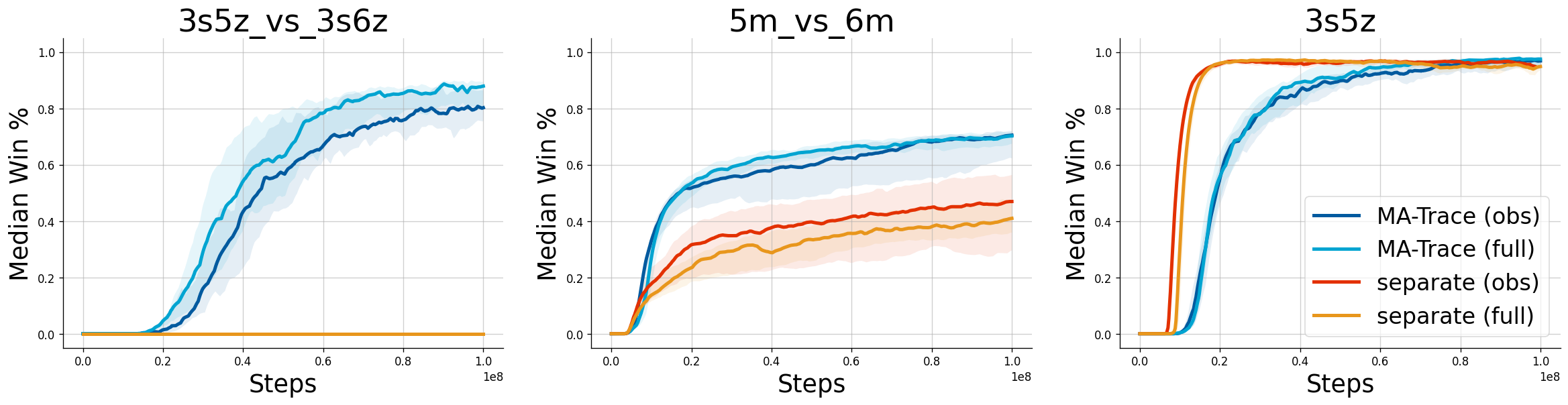}
    \caption{\small Performance of shared (standard \mt{}) and separate agents' networks.}
    \label{fig:separate_ablation}
\end{figure}

\begin{figure}[h!]
    \centering
    \includegraphics[width=\linewidth]{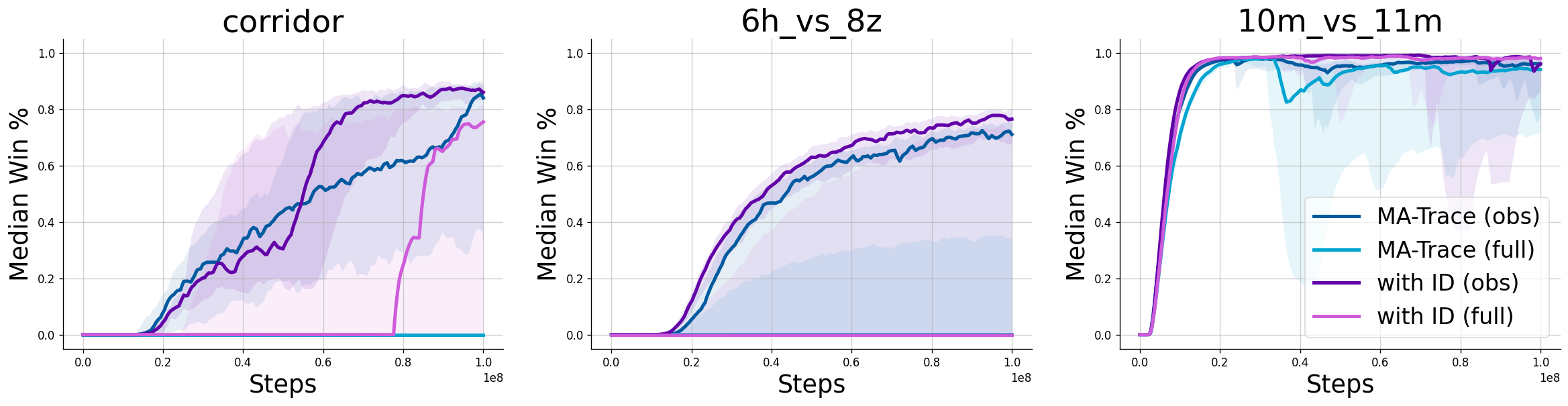}
    \caption{\small The impact of enriching observation with agent ID.}
    \label{fig:ID_ablation}
\end{figure}

\section{Limitations and further work}\label{sec:limitations}

We show that \mt{} successfully solves SMAC tasks. Further benchmarking is needed to underpin its quality. This includes testing on more environments, both fully cooperative (like SMAC) and competitive. The latter might require further algorithmic developments. 

The importance sampling weights successfully reduce distributional shifts arising in distributed training. An interesting question is whether they can also reduce the shifts introduced by the non-stationarity of the multi-agent environment.

\mt{} exhibits lower sample efficiency than the other methods we used for comparisons. This, at least partially, can be explained by its on-policy nature. Adapting the importance correction to accommodate more off-policy data would be an important achievement.

\section{Conclusions}\label{sec:conclusions}

In our work, we introduced \mt{}, a new multi-agent reinforcement learning algorithm. We evaluated it on $14$ scenarios constituting the StarCraft Multi-Agent Challenge and confirmed its strong performance. We also included ablations regarding importance sampling, centralization of learning, scaling, and sharing of parameters.
 
Thanks to the use of importance weights, \mt{} is highly scalable. Furthermore, the convergence properties of our algorithm highlighted by Theorem \ref{thm:fixed_point_theorem_abstract} show that it has not only experimental but also mathematical grounding. 

\begin{acks}
  We thank Konrad Staniszewski for discussions and experiments in the prequel of this project. 
  
  The work of Michał Zawalski was supported by the Polish National Science Center (NCN) grant UMO-2019/35/O/ST6/03464.  The work of
  Henryk Michalewski was supported by the Polish National Science Center
  grant UMO-2018/29/B/ST6/02959. 
  The work of Piotr Miłoś was supported by the NCN grant UMO-2017/26/E/ST6/00622. This research was supported by the PL-Grid Infrastructure. Some experiments were performed using the Entropy cluster funded by
  NVIDIA, Intel, the Polish National Science Center grant UMO-2017/26/E/ST6/00622 and  ERC Starting Grant TOTAL.  Our experiments were managed using \url{https://neptune.ai}. We would like to thank the Neptune team for providing access to the team version and technical support.
  We gratefully acknowledge Polish high-performance computing infrastructure PLGrid (HPC Centers: ACK Cyfronet AGH, PCSS) for providing computer facilities and support within computational grant no. PLG/2021/014561.
\end{acks}



\bibliographystyle{ACM-Reference-Format} 
\bibliography{bibliography}


\appendix
\newpage
\section{The proof of the fixed point theorem} \label{sec:proof}
\newcommand{\RR}{\mathcal{R}}
Let $\mathcal{C}_\infty^K$ be the space of functions $\mathcal{S}\mapsto \mathbb{R}^K$ endowed with the infinity norm $\|f\|_{\mathcal{C}_\infty^K} = \sup_{s\in \mathcal{S}} |f(s)|_\infty$. $\mathcal{C}_\infty^K$ is Banach if $\mathcal{S}$ is finite or we additionally assume that functions are continuos (this encompasses the case when $\mathcal{S}$ is discrete).

For an easy reference we recall the statement of Theorem \ref{thm:fixed_point_theorem_abstract}. Let $c:\mathcal{S}\times\mathcal{A}\to\mathbb{R}_{+}$, $\rho:\mathcal{S}\times\mathcal{A}\to\mathbb{R}_{+}$ be {measurable} functions and set $c_{t}=c(s_{t},a_{t})$, $\rho_{t}=\rho(s_{t},a_{t})$.

Define the operator $\mathcal{R}: \mathcal{C}_\infty^K \mapsto \mathcal{C}_\infty^K$ by
\begin{equation}
\begin{split}
  \mathcal{R}V(s)&:=V(s)+\\
  &+\E_{\mu}{\left[\sum_{t=0}^{+\infty}{\gamma^{t}(c_{0}\cdots c_{t-1})\rho_{t}(r_{t}+\gamma V(s_{t+1})-V(s_{t}))}|s_{0}=s\right]}.
\end{split}
\end{equation}

\begin{theorem} 
  Let $c, \rho$ be such that for any $s \in \mathcal{S}, a\in \mathcal{A}$
\begin{equation} \label{eq:alphadefintion}
  \alpha(s,a):=\rho(s,a)-c(s,a)\E_{a'\sim\mu(\cdot|s')}\left[\rho(s',a')\right]\geq0,
\end{equation}
where $s'$ is the state obtained from $s$ after issuing action $a$.  Assume also that $\E_{\mu}\rho_{0}\geq\beta\in(0,1]$. Then the operator $\mathcal{R}$ is $\mathcal{C}_{\infty}^K$ contraction with a unique fixed point $V^{\tilde{\pi}}$ which is a value function of a policy $\tilde{\pi}$ given by
\begin{equation}\label{eq:policy_after_corrections_appendix}
  \tilde{\pi}(a|s):=\frac{\rho(s,a)\mu(a|s)}{\sum_{b\in\mathcal{A}}\rho(s,b)\mu(b|s)}. 
\end{equation}
The contraction constant $\eta$ is smaller than $1-(1-\gamma)\beta<1$.
\end{theorem}
\begin{remark}\label{eq:original_proof_glitch}
  We note that the proof \cite[Theorem 1]{DBLP:conf/icml/EspeholtSMSMWDF18} has some glitches. A careful examination reveals that the analysis of the first displayed equation on page 12 in \cite{DBLP:conf/icml/EspeholtSMSMWDF18} is not correct. We fix it by making introducing filtration and analyzing the measurability of explicitly and argue that such an approach makes the proof more clear
\end{remark}

\begin{proof}
  Let $\tilde{c}_{t}:=\prod_{u=0}^{t}c_{u}$ with the convention $\tilde{c}_{t}=1$ for $t\leq0$. We use the same convention for $\rho_{s}$. It is convenient to express $\RR$ in the following form
  \begin{align*}
    \RR V(s)=&(1-\E_{\mu}\rho_{0})V(s)+\\
    +&\E_{\mu}\left[\sum_{t=0}^{+\infty}\gamma^{t}\tilde{c}_{t-1}\left(\rho_{t}r_{t}+\gamma[\rho_{t}-c_{t}\rho_{t+1}]V(s_{t+1})\right)\right].
  \end{align*}
  Fix $V_{1},V_{2} \in \mathcal{C}_\infty^K$ and put $\mathcal{C}_\infty^K\ni \Delta V:=V_{1}-V_{2}$. For $s \in \mathcal{S}$ we write
  \begin{align*}
    \RR V_{1}(s)-\RR V_{2}(s)=&(1-\E_{\mu}\rho_{0})\Delta V(s)+\\
    +&\E_{\mu}\left[\sum_{t=0}^{+\infty}\gamma^{t+1}\tilde{c}_{t-1}[\rho_{t}-c_{t}\rho_{t+1}]\Delta V(s_{t+1})\right].
  \end{align*}
We define filtration $\left\{ \mathcal{F}_{t}\right\} _{t=0}^{+\infty}$ by
\begin{equation*}
  \mathcal{F}_{t}:=\sigma((s_{0},a_{0},\ldots,a_{t-1},s_{t})).
\end{equation*}
We rewrite in terms of conditional expectation
\begin{align*}
  a:=&\E_{\mu}\left[\tilde{c}_{t-1}[\rho_{t}-c_{t}\rho_{t+1}]\Delta V(s_{t+1})\right]\\
  =&\E_{\mu}\left(\E_{\mu}\left[\tilde{c}_{t-1}[\rho_{t}-c_{t}\rho_{t+1}]\Delta V(s_{t+1})\vert\mathcal{F}_{t+1}\right]\right).
\end{align*}
Clearly, all terms except for $\rho_{t+1}$ are $\mathcal{F}_{t+1}$ measurable, thus $$a=\E_{\mu}\left[\tilde{c}_{t-1}[\rho_{t}-c_{t}\E_{\mu}\left(\rho_{t+1}\vert\mathcal{F}_{t+1}\right)]\Delta V(s_{t+1})\right].$$
Recall \eqref{eq:alphadefintion} and observe that $\rho_{t}-c_{t}\E_{\mu}\left(\rho_{t+1}\vert\mathcal{F}_{t+1}\right)=\alpha(s_{t},a_{t})=:\alpha_{t+1}$. Let us also put by convention $\alpha_{0}=1-\E_{\mu}\rho_{0}$. Thus shifting indices we get 
\begin{equation*}
  \RR V_{1}(s)-\RR V_{2}(s)=\E_{\mu}\left[\sum_{t=0}^{+\infty}\gamma^{t}\tilde{c}_{t-2}\alpha_{t}\Delta V(s_{t})\right].
\end{equation*}
By our assumptions we have $\tilde{c}_{t},\alpha_{t}\geq0$ and thus we get 
\begin{equation*}
  |\RR V_{1}(s)-\RR V_{2}(s)|
  \leq\Vert V_{1}-V_{2}\Vert_{\mathcal{C}_{\infty}^K}\E_{\mu}\left[\sum_{t=0}^{+\infty}\gamma^{t}\tilde{c}_{t-2}\alpha_{t}\right].
\end{equation*}
We are left with rather straightforward calculations (assume by convention that $\rho_{-1}=1$).
\begin{align*}
  \E_{\mu}\left[\sum_{t=0}^{+\infty}\gamma^{t}\tilde{c}_{t-2}\alpha_{t}\right]	&=\E_{\mu}\left[\sum_{t=0}^{+\infty}\gamma^{t}\tilde{c}_{t-2}\rho_{t-1}\right]-\E_{\mu}\left[\sum_{t=0}^{+\infty}\gamma^{t}\tilde{c}_{t-1}\E_{\mu}\left(\rho_{t}\vert\mathcal{F}_{t}\right)\right]\\
	&=\E_{\mu}\left[\sum_{t=0}^{+\infty}\gamma^{t}\tilde{c}_{t-2}\rho_{t-1}\right]-\E_{\mu}\left[\sum_{t=0}^{+\infty}\gamma^{t}\tilde{c}_{t-1}\rho_{t}\right]\\
	&=1+\gamma\E_{\mu}\left[\sum_{t=0}^{+\infty}\gamma^{t}\tilde{c}_{t-1}\rho_{t}\right]-\E_{\mu}\left[\sum_{t=0}^{+\infty}\gamma^{t}\tilde{c}_{t-1}\rho_{t}\right]\\
	&=1+(\gamma-1)\E_{\mu}\left[\sum_{t=0}^{+\infty}\gamma^{t}\tilde{c}_{t-1}\rho_{t}\right]\leq1+(\gamma-1)\beta.
\end{align*}
The last inequality follows by dropping all summands except $t=0$.

Now, we are to determine the unique fixed point. Recall \eqref{eq:policy_after_corrections_appendix}, we calculate 
\begin{align*}
  &\E_{\mu}\left[\rho_{t}\left(r_{t}+\gamma V^{\tilde{\pi}}(s_{t+1})-V^{\tilde{\pi}}(s_{t})\right)|s_{t}\right]\\
  &=c\E_{\tilde{\pi}}\left[\rho_{t}\left(r_{t}+\gamma V^{\tilde{\pi}}(s_{t+1})-V^{\tilde{\pi}}(s_{t})\right)|s_{t}\right]\\
  &=0,
\end{align*}
where $c=\sum_{b\in\mathcal{A}}\rho(s,b)\mu(b|s)$ is the normalizing constant and the second equality hold by the Bellman equation. 
\end{proof}

\newpage
\section{Centralized training and decentralized execution}\label{sec:decentralized_appendix}

In a multi-agent problem, the agents take actions based on their local observations $s_i$, according to their policies $\pi_i$. {\bf Decentralized training} is a simple approach of learning $\pi_i$, using any (single-agent) RL algorithm separately for each agent. To better utilize the structure of a multi-agent problem, a paradigm of {\bf centralized training and decentralized execution} was introduced. The key idea is to centralize the learning process of all the agents -- use shared knowledge provided by observations and any other information available during training to more effectively optimize the decentralized policies.
It is now considered a leading paradigm and is usually chosen off-the-shelf.

Recently the authors of \cite{lyu2021contrasting} suggested that this should be reinvestigated.
Sharing the knowledge in centralized training is indeed beneficial, but the paradigm of independent training also has advantages. The multi-agent problems lack the strong theoretical guarantees associated with the standard case. For example, if we train all the agents simultaneously, the environment changes during training from every agent's perspective.
Such non-stationarity can destabilize standard algorithms.
According to \cite{lyu2021contrasting}, decentralized learning partially mitigates this issue -- for every agent, a separate value network is trained; thus it averages much of the stochasticity in the environment, producing more stable estimates. On the other hand, because of the limited information, these values are less accurate. Therefore choosing between centralized and decentralized training is a tradeoff.

To address these issues, we compared \mt{} with its decentralized version, \imatrace{}.
However, in the StarCraft Multi-Agent Challenge, the decentralized version learns much slower and fails to reach good performance; see Figure~\ref{fig:appendix_decentralized}.

\begin{figure}[h!]
    \centering
    \includegraphics[width=\linewidth]{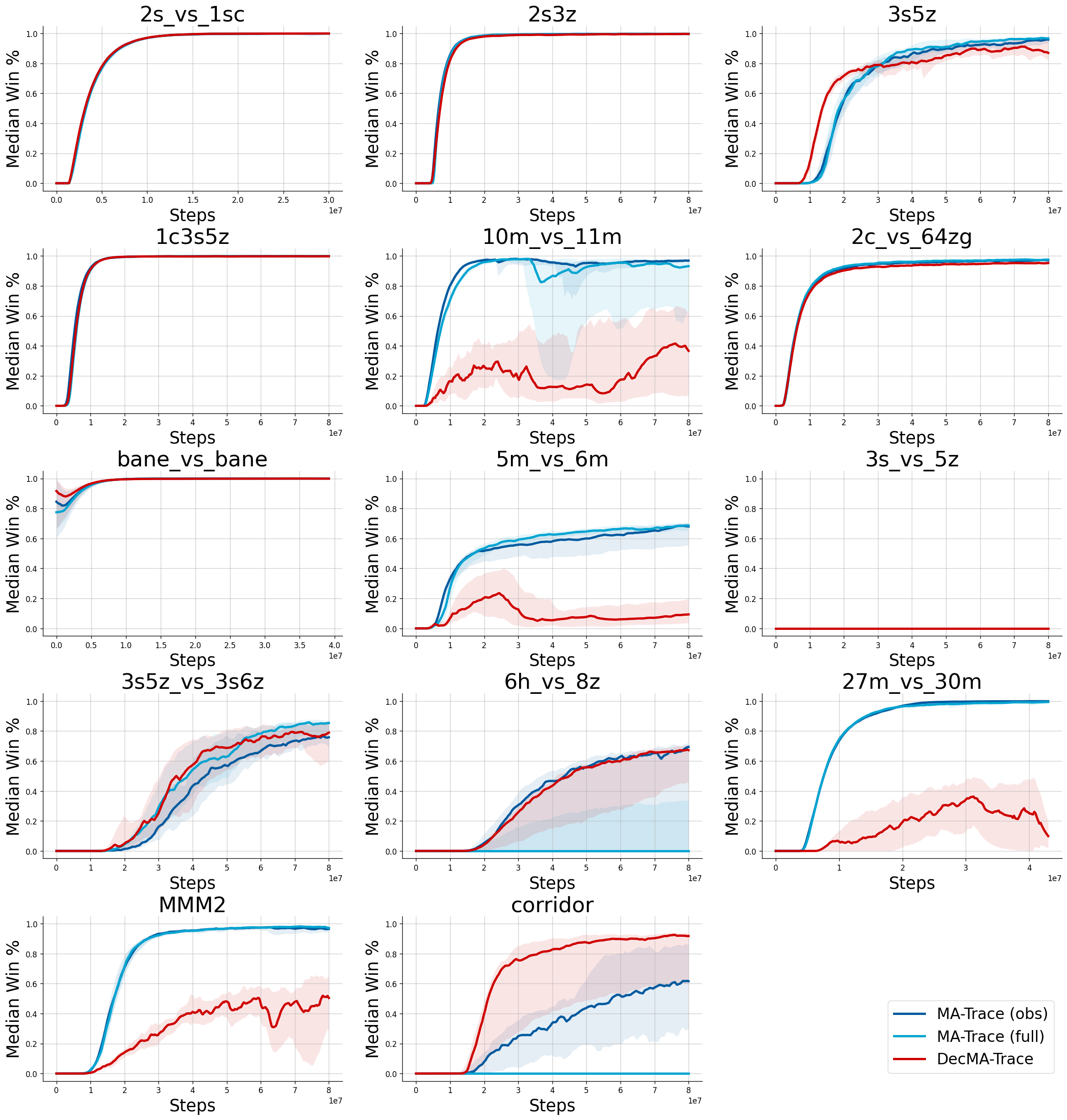}
    \caption{\small Comparison of \mt{} and \imatrace{} on the StarCraft Multi-Agent Challenge tasks.}
    \label{fig:appendix_decentralized}
\end{figure}

\newpage
\section{Networks' architectures}\label{sec:architectures_appendix}
In our experiments, we use feed-forward networks with two hidden layers of $64$ neurons and ReLU activations (without normalization).
This is smaller than the networks used in benchmark \cite{yu2021benchmarking} and \cite{DBLP:journals/corr/abs-2003-08839}, but we found this sufficient to obtain good results. In this work, we focus on algorithmic aspects rather than tuning architectures. However, for completeness, we include a discussion of some most popular design choices.

\subsection{Sharing parameters}\label{sec:separate_networks_exps}

Our default (and most effective) scheme uses a single shared network for the actors and a separate one for the critic. We experimented with sharing feature extractors between the agents and the critic, which performed worse. We also checked the performance when training separate networks for all the components. Interestingly, with this approach, the learning is much faster on the easy tasks, such as \textit{3s5z} or \textit{2c\_vs\_64zg}, but completely fails on the hardest ones; see Figure~\ref{fig:appendix_separate}.

\begin{figure}[h]
    \centering
    \includegraphics[width=\linewidth]{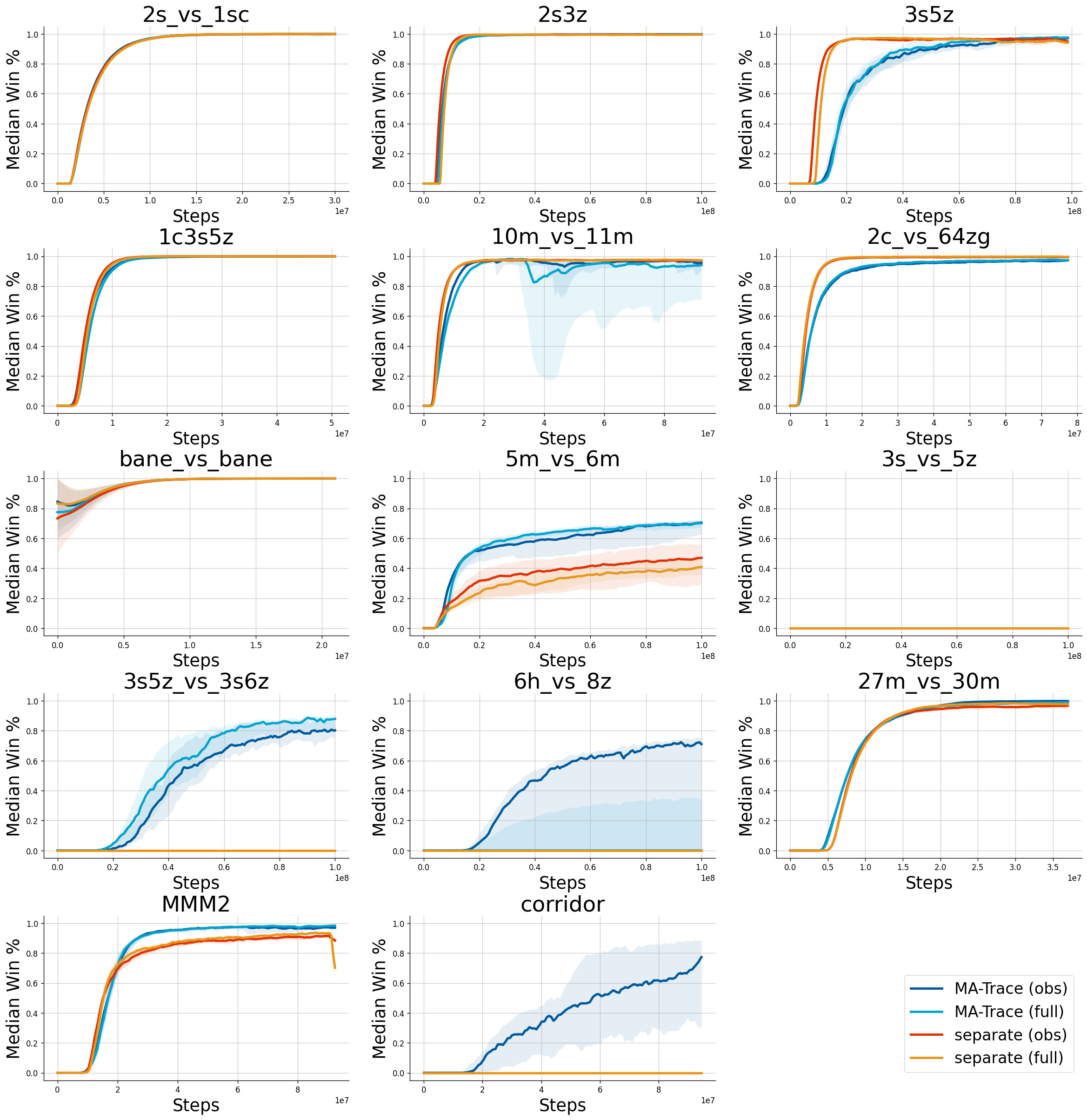}
    \caption{\small Comparison of standard \mt{} (with a shared actor network) and its ablation with separate networks.}
    \label{fig:appendix_separate}
\end{figure}

\subsection{ID experiments}\label{sec:id_experiments_appendix}
Sharing the agents' networks, as described in Section \ref{sec:separate_networks_exps}, can lead to poor results if agents are not homogenous (for example, the behavior of shooting units should differ from melee units). This might be circumvented by adding units' characteristics (which is a default in SMAC) or enriching the observations with one-hot encoded ID. In Figure~\ref{fig:appendix_id} we present the results of experiments in which we use these two mechanisms. We observe some improvements, which are, however, relatively minor. 

\begin{figure}[h]
    \centering
    \includegraphics[width=\linewidth]{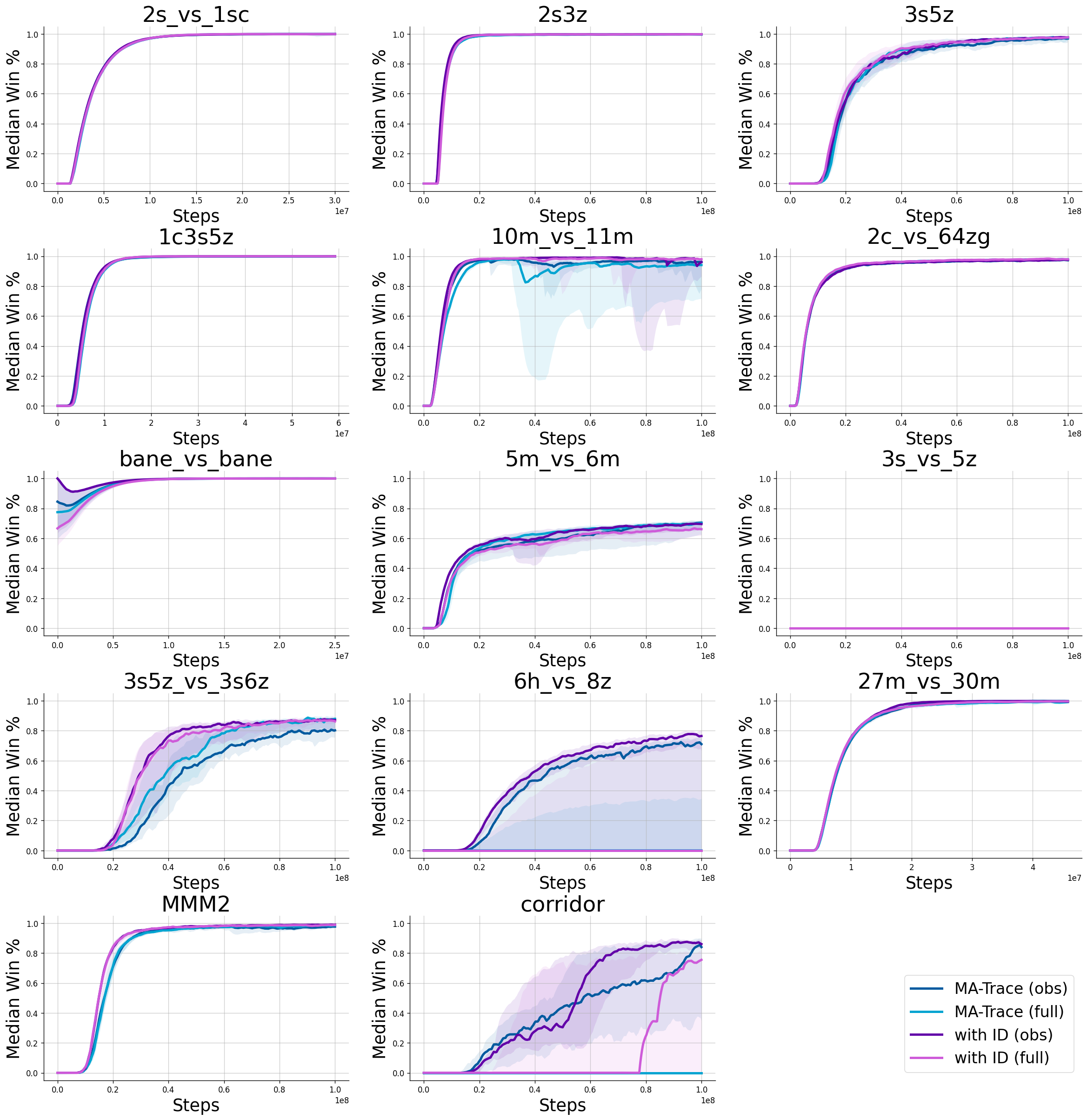}
    \caption{\small Comparison of \mt{} conditioned on observation with and without (default) agents' IDs.}
    \label{fig:appendix_id}
\end{figure}

\subsection{Agents with memory}\label{sec:appendix_memory}
The units in StarCraft Multi-agent Challenge have limited sight range, thus the environment is partially observable.
A common approach to mitigate such an issue, is to give the agents memory, e.g. by using recurrent networks or frame stacking.
The authors of \cite{DBLP:journals/corr/abs-2003-08839} show that using recurrent policy is required to achieve good results in the hardest tasks.
However, our experiments show that memoryless policies can be successfully trained to solve all the tasks (possibly except \textit{3s\_vs\_5z}, see Section~\ref{section:3svs5z}).

To implement a simple memory it is enough to pass a few previous observations concatenated (``framestack'').
We experimented with both observation-based and state-based critics.
Figures \ref{fig:appendix_stacking_obs} and \ref{fig:appendix_stacking_full} show the progress of training memory-equipped agents. One can observe hardly any difference in most tasks. On the easy task \textit{2c\_vs\_64zg}, using memory leads to faster training.
What is particularly interesting, on the \textit{corridor}, task the stacked versions learn much faster.
This may be related to the strategy learned by \mt{}, in which at some point a group of units has to hide and wait (see Section \ref{section:corridor} for detailed description). Execution of this strategy is probably easier when using at least short memory. However, on some other tasks, such as \textit{5m\_vs\_6m}, training with memory is much less stable and effective.

\begin{figure}[h]
    \centering
    \includegraphics[width=\linewidth]{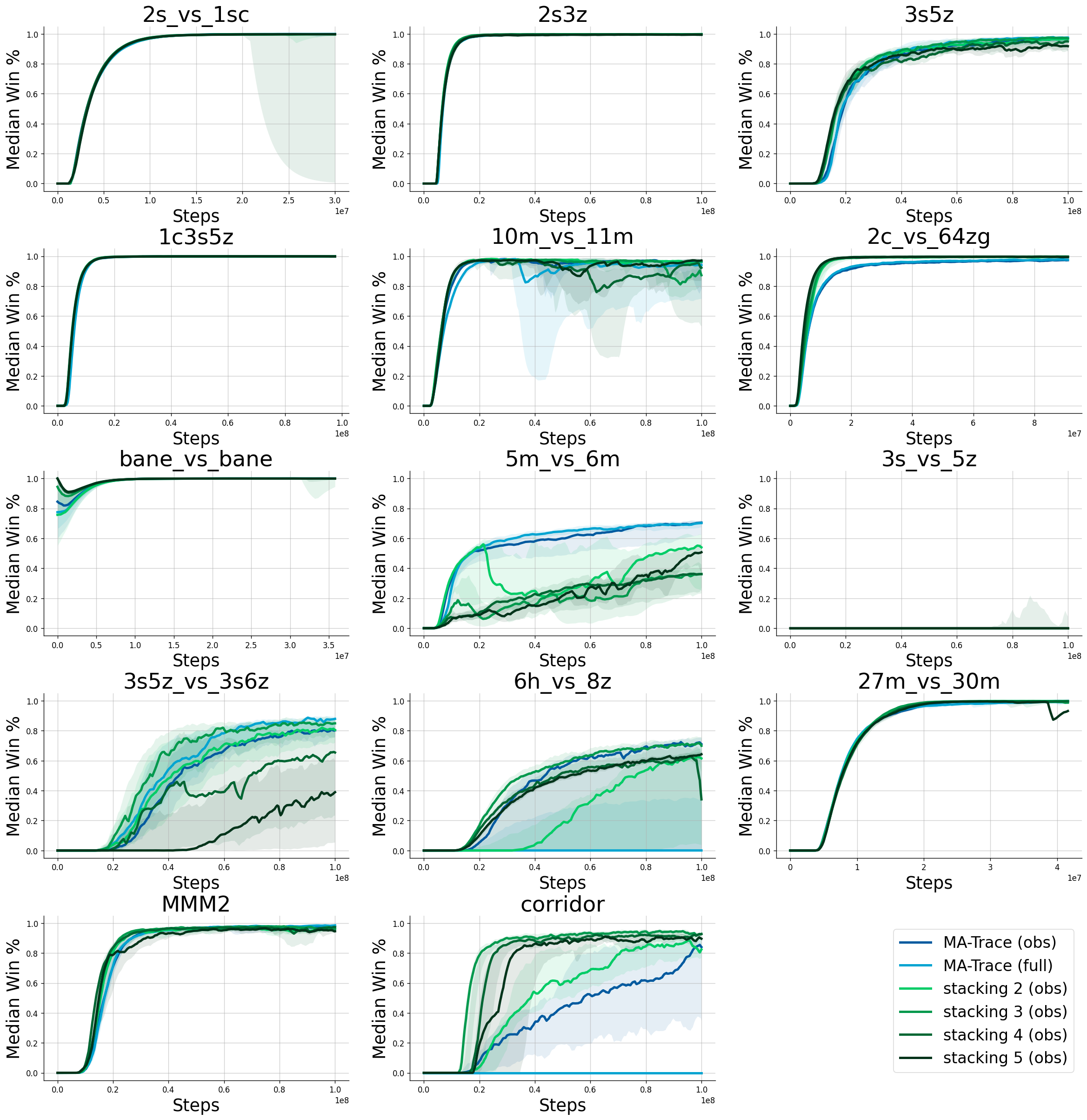}
    \caption{\small Comparison of \mt{} with growing number of stacked frames. The critic networks in ablations are conditioned on aggregated observations.}
    \label{fig:appendix_stacking_obs}
\end{figure}

\begin{figure}[h]
    \centering
    \includegraphics[width=\linewidth]{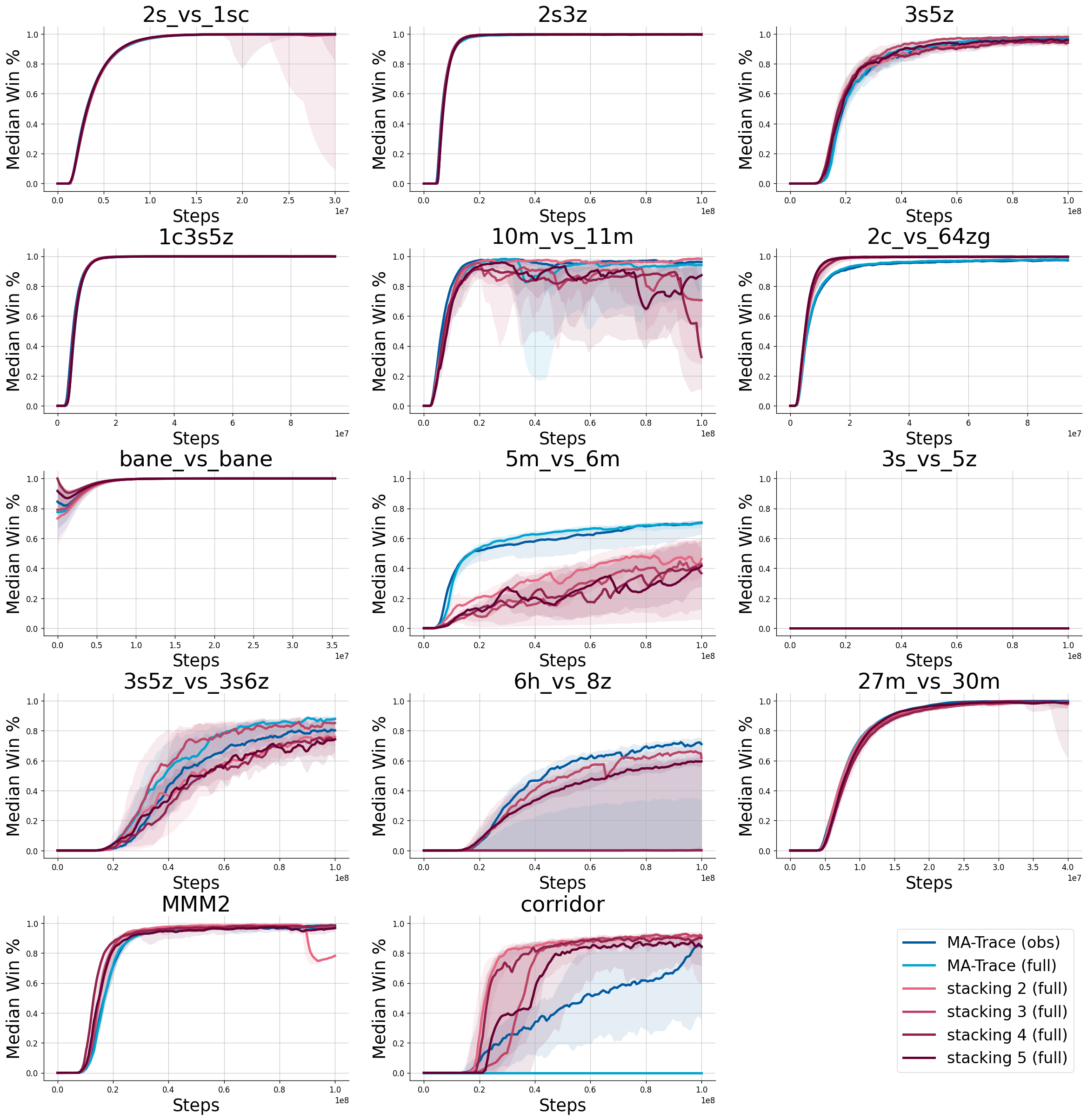}
    \caption{\small Comparison of \mt{} with growing number of stacked frames. The critic networks in ablations are conditioned on environment state.}
    \label{fig:appendix_stacking_full}
\end{figure}

\newpage
\section{Training details}\label{sec:training_appendix}
\subsection{SEED RL} \label{section:seed_rl}
We base our code on SEED RL \cite{espeholt2019seed}, which is an open-source framework for distributed learning, licensed under the Apache License, Version 2.0.
This library provides an implementation of V-Trace for single-agent environments.

\subsection{Hyperparameters} \label{sec:hiperparemters}
In our experiments the parameters crucial to the final performance were: \textit{learning rate} and \textit{entropy cost}.
For other hyperparameters, we use the default values provided by SEED.
The most relevant, common for all the experiments, are listed in Table~\ref{tab:default_params}.

\begin{table}[h]
\setlength{\tabcolsep}{2em}
\centering
\begin{tabular}{cc}
	\toprule
	hyperparameter & value \\
	\midrule
	batch size	& $32$ \\
	optimizer	& Adam \\
	gamma		& $0.99$ \\
	$\rho$		& $1.0$ \\
	$c$		& $1.0$ \\
	$\lambda$	& $1.0$ \\
	initial entropy cost	& $1.0$ \\
	target entropy	& $10^{-5}$ \\
	\bottomrule
\end{tabular}
\smallskip
\caption{Default parameters for our experiments.}
\label{tab:default_params}
\end{table}

We searched for the best learning rate in the set $[3.5\cdot10^{-3}, 2.5\cdot10^{-3}, 1.5\cdot10^{-3}, 10^{-3}, 7\cdot10^{-4}, 5\cdot10^{-4}, 3.5\cdot10^{-4}]$.
We experimented with adding learning rate schedules with warmup and decay, though no such scheme appeared to be beneficial.

Another hyperparameter that has a strong influence on the final results is \textit{entropy cost}.
By default, SEED sets it to a constant value of $2.5\cdot10^{-4}$ and allows to use annealing.
We found that aggressive exploration at the beginning is crucial to reach good results. We annealed \textit{entropy cost} from $1$ towards $10^{-5}$. The speed of adjustment was tuned in the set $[10, 5, 2.5, 1, 0.5]$.

In Table~\ref{tab:specific_parameters}, we show the best parameters for our experiments, found by the above grid-search.
Unlike in some other works, we found no advantage of using gradient clipping; thus, we leave the gradients not clipped.

\begin{table}[h]
\centering
\begin{tabular}{ccc}
	\toprule
	hyperparameter 		& \mt{} (obs) 	& \mt{} (full) \\
	\midrule
	learning rate		& $10^{-3}$ 	& $7\cdot10^{-4}$ \\
	entropy adjustment	& 10 		& 10 \\
	\bottomrule
\end{tabular}
\smallskip
\caption{Specific parameters for our experiments.}
\label{tab:specific_parameters}
\end{table}

\subsection{Infrastructure used} \label{sec:infrastructure_used}
The typical configuration of a computational node used in our experiments was: the Intel Xeon E5-2697 $2.60$GHz processor with $128$GB memory. On a single node, we ran one experiment with $30$ workers.
A typical experiment was run for about $20$h. For the final evaluation, we extended the training to 3 days, which is usually equivalent to about $3\cdot10^8$ environment steps. We did not use GPUs; we found that with the relatively small size of the network it offers only slight wall-time improvement while generating substantial additional costs. 

\newpage
\section{Ablations}\label{sec:ablations_appendix}

\subsection{Critic comparison} \label{sec:critic_comparison_appendix}

During centralized training, the critic network in \mt{} algorithm uses any information available. A natural choice is to aggregate the observations available to the agents, and we denote this version as \mt{} (obs). This might not be a sufficient statistic of the (Markov) state of the environment. SMAC provides additional access to such a state description, which we use in \mt{} (full). Intuitively, using complete information should be advantageous.

As shown in Figure~\ref{fig:all_training_curves}, on most tasks, both the versions do not differ much.
On \textit{3s5z\_vs\_3s6z} there is a small advantage on the full-state side.
However, \mt{} (full) shows little progress on \textit{6h\_vs\_8z} and \textit{corridor}, as opposed to \mt{} (obs).
Therefore we consider \mt{} (obs) to be our default version.

Such behavior is a bit counter-intuitive. We speculate that some information available in the agents' observation is not easily accessible (computable) for the full state. To verify this, we compared the two versions with another, \mt{} (obs+full), which uses both the aggregated observations and the full state. As we can see in Figure~\ref{fig:all_training_curves}, it trains similarly to \mt{} (obs), without significant advantage on any task. Moreover, on the hardest tasks, the training progresses a bit slower. We leave a precise explanation of this behavior as future work.

\subsection{Importance sampling ablation}\label{sec:is_ablation_appendix}
We claim that the strong performance of \mt{} is due to using importance weights. To verify this statement, we compared \mt{} with its ablation without importance weights.
The training curves for all the tasks are shown in Figure~\ref{fig:appendix_nois}.
Without the corrections, the training is unstable and reaches good performance only on the easiest tasks.
The difference gets even larger when using more compute workers.
One notable exception is the \textit{corridor} task, in which the ablated version trains faster than \mt{}, though eventually both reach similar results.

\begin{figure}[h]
    \centering
    \includegraphics[width=\linewidth]{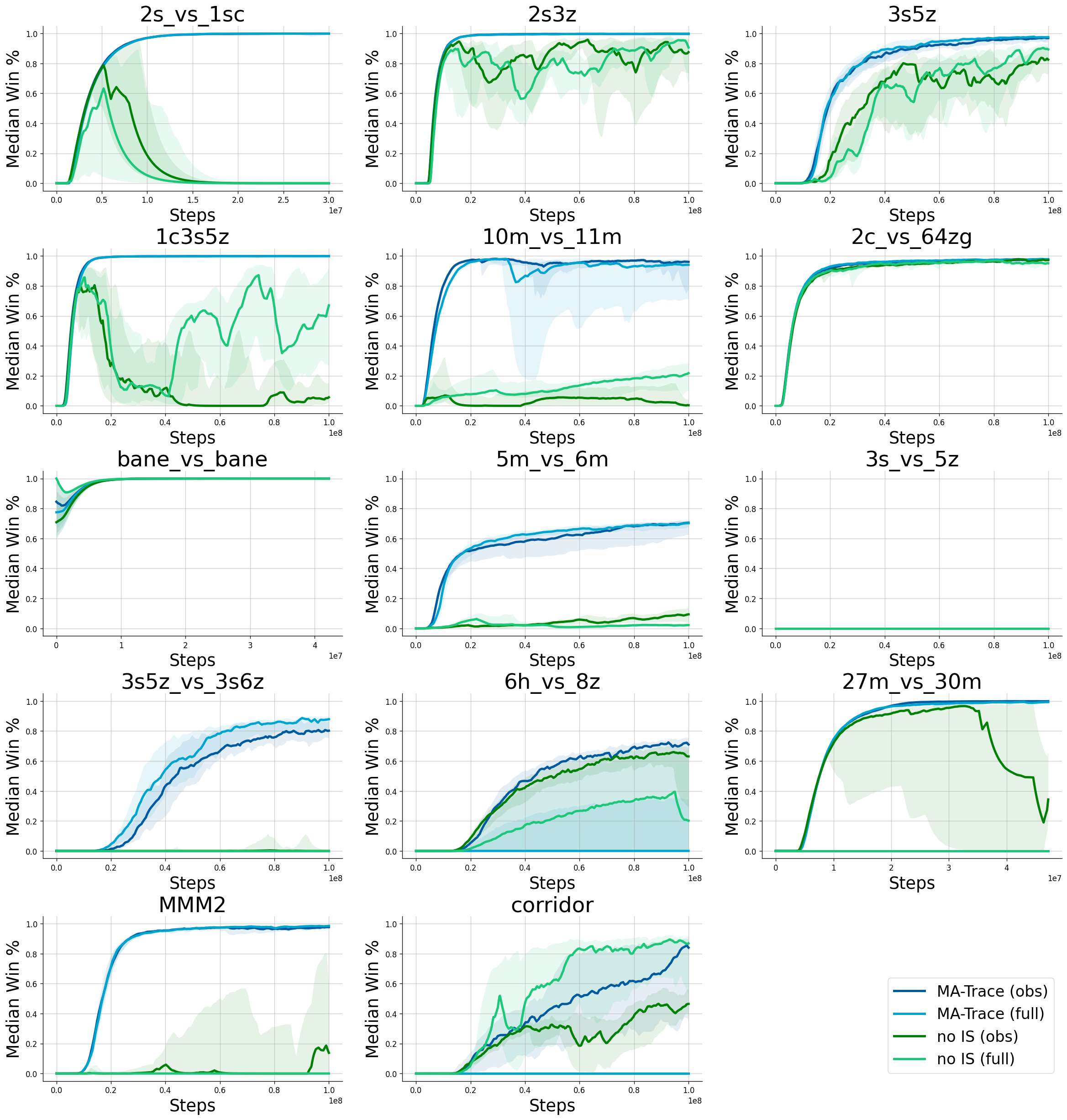}
    \caption{\small Comparison of \mt{} with and without importance sampling.}
    \label{fig:appendix_nois}
\end{figure}

\subsection{Scaling experiments}\label{sec:scaling_experiments}
\piotrm{Why do we need this section}\michalz{It is referred in the main part. However, seems we have nothing more to say, can be removed.}
\mt{} utilizes importance sampling to reduce the shifts arising naturally in distributed training (when behavioral policies lag behind the target policy). This allows our algorithm to be highly scalable.
Our experiments confirm that \mt{} scales almost linearly on at least $70$ workers.

As shown in Section~\ref{sec:is_ablation_appendix}, the importance weights indeed enable stable training, even when distributed to many workers.
What is equally important, our experiments show that scaling does not affect the learned policies significantly.

\newpage
\section{StarCraft Multi-Agent Challenge} \label{sec:smac_appendix}
To evaluate \mt{}, we use StarCraft Multi-Agent Challenge (SMAC) \cite{samvelyan19smac}.
It is a standard benchmark for multi-agent algorithms, used e.g. by \cite{DBLP:journals/corr/abs-2003-08839,6_DBLP:conf/aaai/FoersterFANW18, yu2021benchmarking} and many others. The challenge is based on a popular real-time strategy game StarCraft II and consists of $14$ micromanagement tasks. Each task is a small arena in which two teams, controlled by the player by the built-in AI, fight against each other. The goal in every task is to defeat (kill) all the enemy units.

Units belong to one of the three races: Protoss, Terran or Zerg. Additionally, they are divided on a number of classes with unique characteristic, such as speed, shooting range, fire power etc. In each turn they can move or attack an enemy in their shooting range. A unit is considered defeated if its health drops to $0$. A defeated unit remain inactive.

SMAC provides a variety of different tasks. In the easier tasks, the opponents control the same forces. Therefore to win such a game, it is enough to coordinate slightly better than the built-in AI. In the harder tasks, however, the computer starts with a stronger squad. This can be a minor advantage, such as in the task \textit{10 marines vs 11 marines}, or quite a big difference. In particularly hard scenarios, such as \textit{corridor}, it is unreasonable for the player to engage in an open fight, so it is essential to develop a long-term strategy to obtain a positional advantage.

\piotrm{The next sentence is in odds with the previous one.}\michalz{Is it? As far as I understand the concern, "microtrick" actually can be a long-term strategy, the "micro"-part does not enforce any kind of locality I think.}
For the hardest tasks, the authors of the challenge consider the so-called microtricks sufficient to win consistently.
However, it appears that \mt{} in some cases develop unexpected strategies, see e.g. Section~\ref{section:corridor}.
What is particularly interesting, our algorithm manages to learn some techniques associated with professional human players, such as focusing fire, withdrawing low-health units, hit-and-run, sacrificing a unit to deceive the opponent and others.

Units in the game have limited sight range, which makes the environment partially observable.
The observations received by individual agents contain information about all the visible units (including themselves) -- their health, energy, position, class, and other relevant features. All the units beyond the sight range are marked as dead (i.e., defeated and invisible units are not distinguished).
It is possible that the aggregated observations do not provide full information, for there can be enemy units hidden beyond the sight range of any ally. Therefore, to facilitate centralized training, SMAC provides additional access to the full state of the environment.

Learning from binary reward (win/lose) is prohibitively hard in most tasks. Therefore SMAC provides dense rewards to enable training. The team receives additional points for damage dealt and defeating a unit. This scheme is arguably natural and leads to successful training. However, in some cases, it might reinforce undesired behaviors, see, e.g., Section~\ref{section:3svs5z}.

\subsection{\textit{3s\_vs\_5z} task}\label{section:3svs5z}
\mt{} masters all the tasks except \textit{3s\_vs\_5z}, see Table \ref{tab:results}.  In that scenario, we control $3$ Stalkers fighting against $5$ Zealots. Stalkers can attack the enemy from a distance; however, they are no match for Zealots in close combat. A strategy to gain an advantage is to shoot the enemies while they are away and flee when they get close.

All the units in this task are Protoss, so they all have protective shields. The shields absorb some amount of damage, until they are down. However, as opposed to regular health, the shields regenerate slowly with time. As dealing damage yields rewards, it might be beneficial to keep enemy alive infinitely. Apparently, \mt{} learns to this strategy.

Figure~\ref{fig:3svs5z} shows an example of the winning rate and episode rewards. As we can see, after a short training, our algorithm wins almost every time. Further, it discovers that the reward scheme can be exploited -- at some point, the mean return increases fast, while the actual win rate decreases and becomes unstable.

\begin{figure}[h!]
    \centering
    \includegraphics[width=0.6\linewidth]{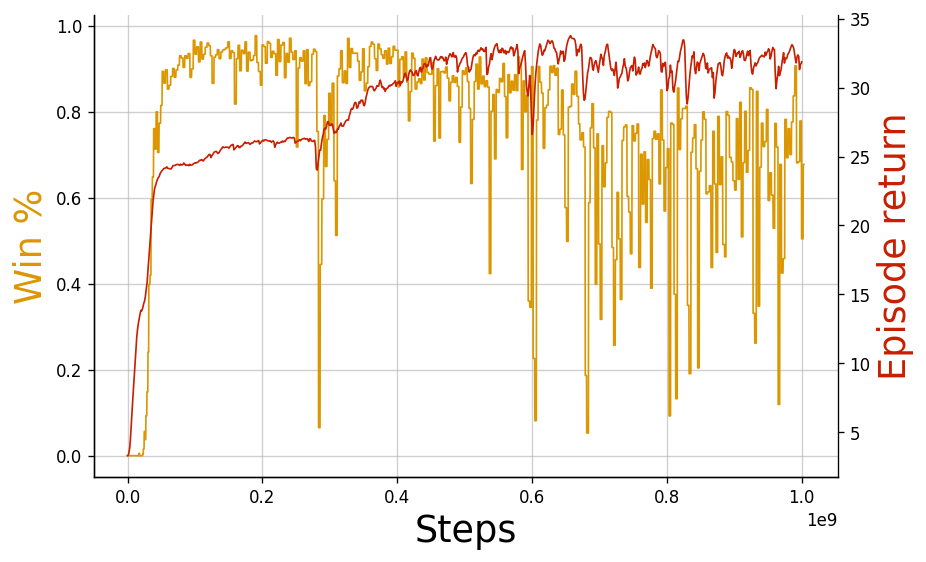}
    \caption{\small The winning rate and episode rewards in a  training for the task \textit{3s\_vs\_5z}.
    }
    \label{fig:3svs5z}
\end{figure}

\subsection{\textit{Corridor} task}\label{section:corridor}

Another \textit{super-hard} scenario (according to \cite{DBLP:journals/corr/abs-2003-08839}) is \textit{corridor}.
In this task, we control a team of $6$ Zealots against $24$ enemy Zerglings.
Though Zealots are far more powerful in combat, they are outnumbered; thus, the open fight is clearly an unreasonable strategy.
However, the fighting arena contains a narrow passage.
The authors of SMAC suggest that a winning strategy is to gather the forces in that passage, where the number of enemy units should lose its importance (possibly inspired by the Battle of Thermopylae).

However, \mt{} develops an alternative interesting strategy. Firstly, our forces split into two groups.
One (strong) hides in the corner, where it easily defeats a few enemies, while the other (one or two units) attracts majority of the enemies to the other side and sacrifice itself.
After defeating the second group, the enemies pass to the far side of the arena and wait, unaware of the hidden group.
Then the strong group attack them from behind and defeat the Zerglings one by one.

The strategy is outlined in Figure \ref{fig:corridor}.

\begin{figure}[h]
\centering

\begin{subfigure}[h]{0.7\linewidth}
	\centering
	\includegraphics[width=0.9\linewidth]{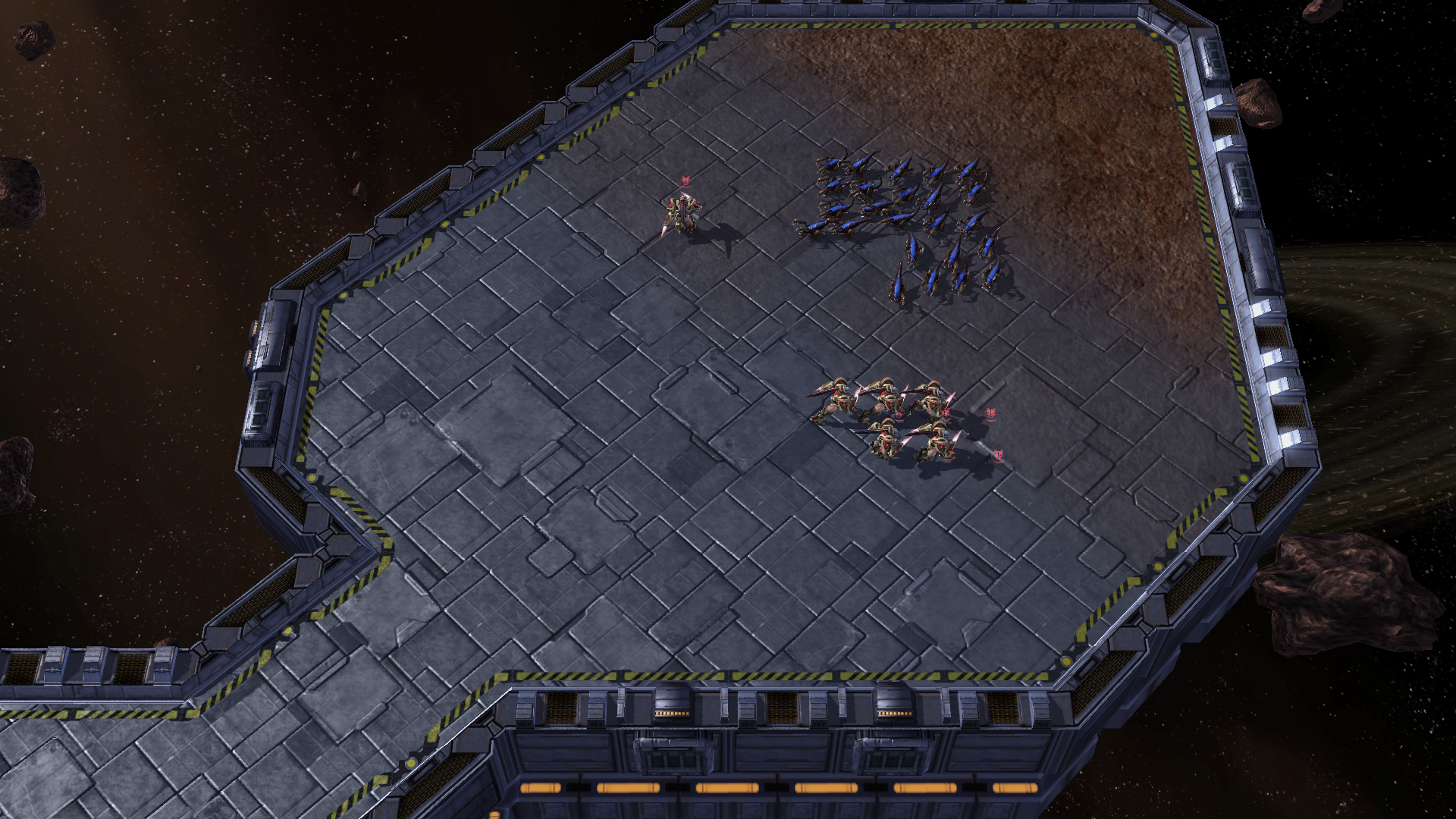} 
	\caption{Split into two groups.}
\end{subfigure}
\begin{subfigure}[h]{0.7\linewidth}
	\centering
	\includegraphics[width=0.9\linewidth]{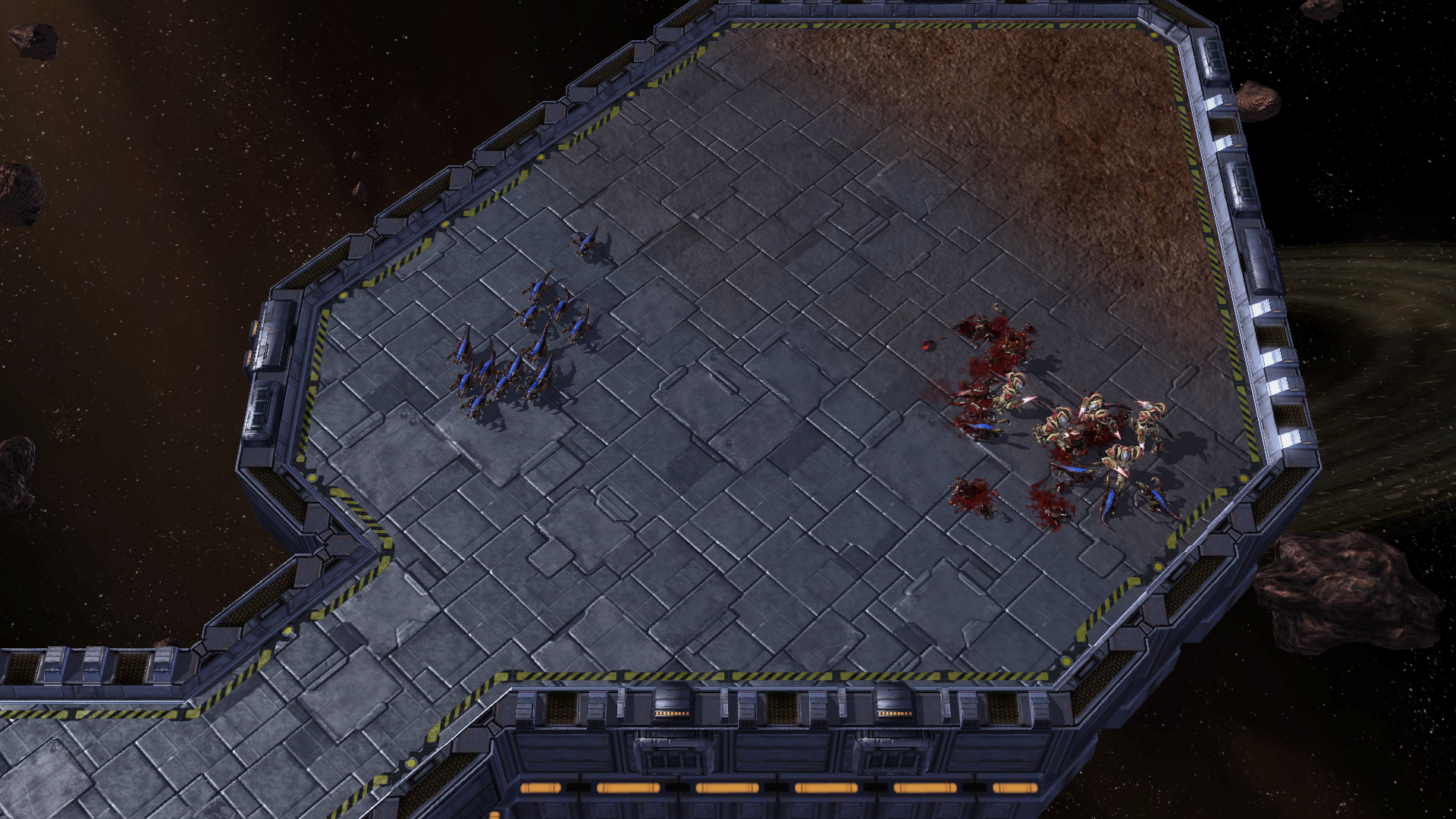}
	\caption{Hide the stronger, sacrifice the weaker.}
\end{subfigure}
\begin{subfigure}[h]{0.7\linewidth}
	\centering
	\includegraphics[width=0.9\linewidth]{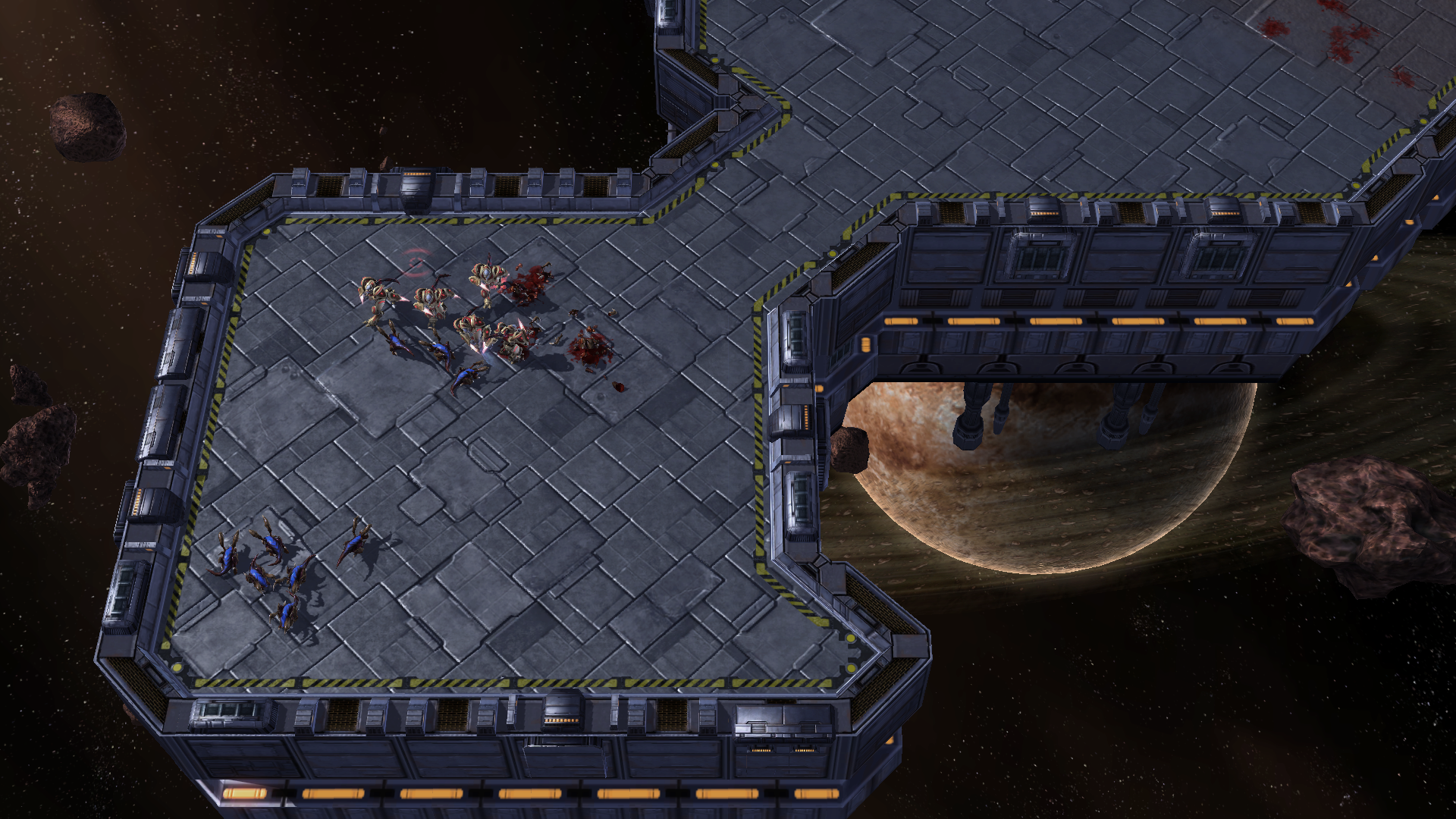}
	\caption{Attack from behind.}
\end{subfigure}

\caption{\small The consecutive parts of strategy executed by \mt{} on the \textit{corridor} task.
The yellow soldiers are our units, while the blue creatures are enemy units.}
\label{fig:corridor}
\end{figure}

\newpage
\section{Full training data}\label{section:all_training_data}
In this section, we provide the main practical results of our work -- complete training data for all standard versions of \mt{} on all the SMAC tasks.
They were trained for three days or until convergence. We report the median win rates and interquartile ranges.

\begin{figure}[h]
    \centering
    \includegraphics[width=\linewidth]{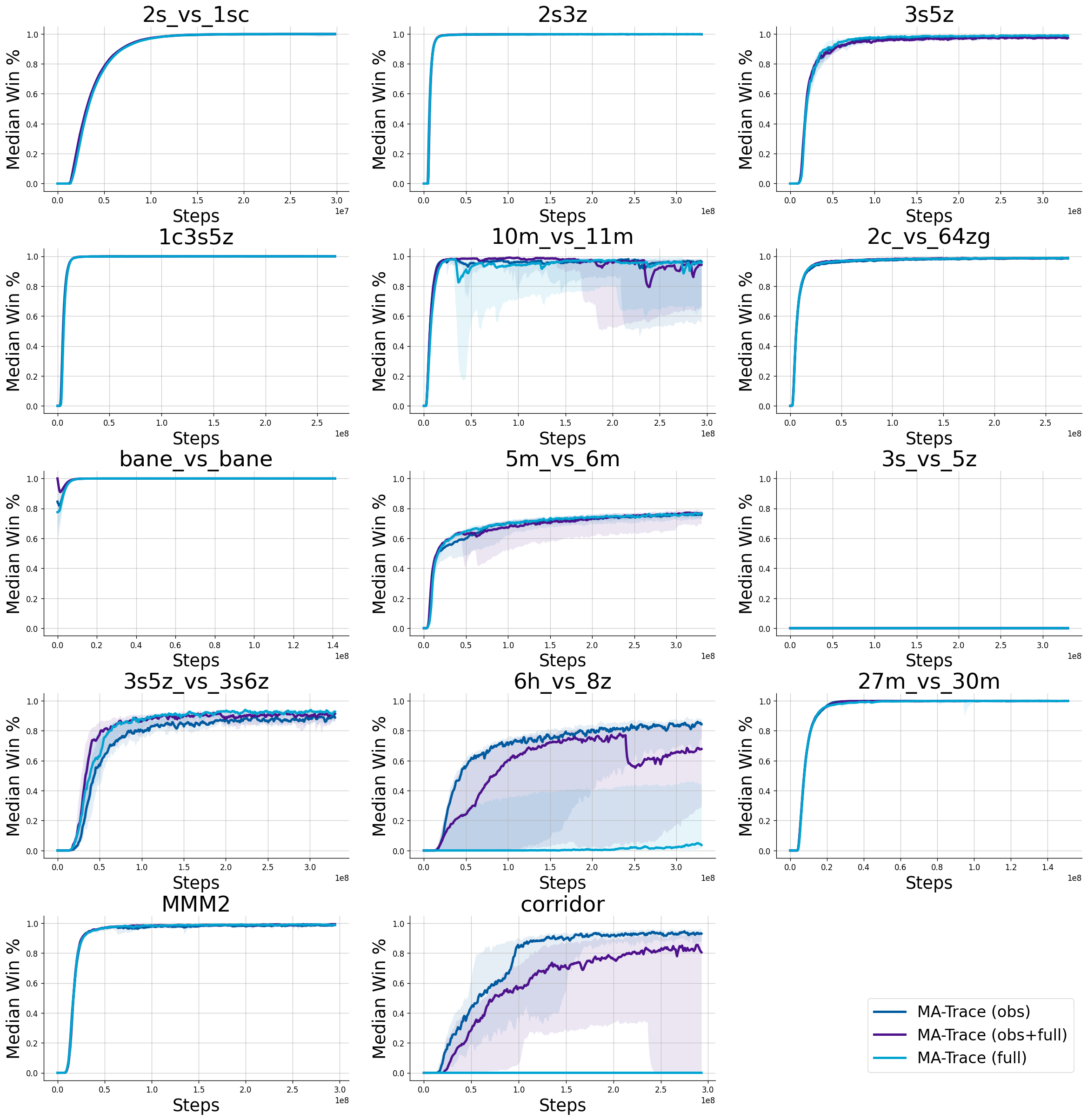}
    \caption{\small Training curves of main \mt{} versions on all the tasks available in StarCraft Multi-Agent Challenge.}
    \label{fig:all_training_curves}
\end{figure}

\end{document}